\documentclass[letterpaper]{article}

\newif\ifdraft
\drafttrue

\usepackage{jair}
\usepackage{times}
\usepackage{helvet}
\usepackage{courier}
\usepackage{fixltx2e} 
\usepackage{url}

\usepackage{paralist}
\usepackage{enumerate}
\usepackage{amsmath}
\usepackage{boxedminipage}
\usepackage[ruled,vlined,linesnumbered]{algorithm2e}

\usepackage{bm}

\usepackage{amsthm}
\usepackage{amsmath}
\usepackage{amssymb}
\usepackage{url}

\usepackage{tikz}
\usetikzlibrary{shapes,decorations,shadows}

\usepackage{enumerate}
\usepackage{paralist}
\usepackage{boxedminipage}
\usepackage[ruled,vlined,linesnumbered]{algorithm2e}
\usepackage{amsthm}
\usepackage{pifont}
\usepackage{tabularx}
\usepackage{amssymb}
\usepackage{stmaryrd}
\usepackage{booktabs}
\usepackage{xspace}
\usepackage{hyperref}
\usepackage{float}
\usepackage[textsize=small]{todonotes}
\usepackage{xcolor}

\usepackage[numbers]{natbib}
\usepackage{hyperref}
\newtheorem{lemma}{Lemma} \theoremstyle{definition}
\newtheorem{definition}{Definition} \theoremstyle{plain}
\newtheorem{proposition}{Proposition} \newtheorem{theorem}{Theorem}
 \theoremstyle{definition}
\newtheorem{example}{Example}

\newcommand{\curlyP}{\mathcal{P}} 
\newcommand{\LC}{\mathsf{LC}}
\newcommand{\str}{\ensuremath{\textit{str}}\xspace}
\newcommand{\Ax}{\ensuremath{\alpha}\space}
\newcommand{\type}{\ensuremath{\textit{type}}}

\newcommand{\types}{\ensuremath{\textit{types}(\T)}\xspace}
\newcommand{\run}{\ensuremath{\textit{frn}}}
\newcommand{\srun}{\ensuremath{\textit{crn}}}
\newcommand{\tail}{\ensuremath{\textit{tail}}}

\newcommand{\suprr}{\ensuremath{\sqsubseteq^*_\T}}

\newcommand{\q}{\ensuremath{\mathfrak{q}}\xspace}

\newcommand{\T}{\mathcal{T}} 
\newcommand{\datalog}{\textsc{Datalog}\xspace}

\newcommand{\ISA}{\ensuremath{\mathbin{\sqsubseteq}}}
\newcommand{\Some}[2]{\ensuremath{\exists{#1}\per{#2}}}
\newcommand{\All}[2]{\ensuremath{\forall{#1}\per{#2}}}

\newcommand{\A}{\mathcal{A}}

\newcommand{\np}{\textsc{NP}\xspace}
\newcommand{\pspace}{\textsc{PSpace}\xspace}

\newcommand{\exptime}{\textsc{ExpTime}\xspace}
\newcommand{\nexptime}{\textsc{NExpTime}\xspace}
\newcommand{\twoexptime}{\textsc{2ExpTime}\xspace}

\newcommand{\nexptimenp}{\textsc{NExpTime}$^{\scriptsize \np}$\xspace}

\newcommand{\rin}{$r\,\in_\T\,\Sigma$\xspace}

\newcommand{\rnotin}{$r\,\not\in_\T\,\Sigma$\xspace}

\newcommand{\ALC}{\ensuremath{\mathcal{ALC}}\xspace}
\newcommand{\ALCHOI}{\ensuremath{\mathcal{ALCHOI}}\xspace}

\newcommand{\per}{.}  \newcommand{\K}{\mathcal K}
\newcommand{\I}{\mathcal I} \newcommand{\J}{\mathcal J}
\newcommand{\tuple}[1]{\langle{#1}\rangle}
\newcommand{\msf}[1]{\mathsf{#1}}

\newcommand{\indivnames}{\msf{N}_{\msf{I}}} 
\newcommand{\conceptnames}{\msf{N}_{\msf{C}}} 
\newcommand{\varnames}{\msf{N}_{\msf{V}}}
\newcommand{\rolenames}{\msf{N}_{\msf{R}}}
\newcommand{\preds}{\msf{N}_{\msf{P}}} \newcommand{\cert}{\msf{cert}}

\newcommand{\AND}{\ensuremath{\sqcap}}

\newcommand{\OR}{\ensuremath{\sqcup}}

\newcommand{\NOT}{\ensuremath{\neg}}
\newcommand{\SOME}{\ensuremath{\exists}}
\newcommand{\ALL}{\ensuremath{\forall}}

\newcommand{\invat}[1]{\hat{#1}}

\newcommand{\vars}{{\ensuremath{\textit{vars}}}}

\newcommand{\dom}[1]{\Delta^{#1}} 
\newcommand{\Int}[2]{#2^{#1}} 

                                                                                   \newcommand{\newt}[1]{{\color{black}#1}}

                                                                                   \newcounter{qcounter}

\begin{document}

\title{Polynomial Rewritings from Expressive Description Logics with Closed Predicates to Variants of Datalog  
}

\author{\name Shqiponja Ahmetaj
 \email  ahmetaj@dbai.tuwien.ac.at \\
 \addr TU Wien, Austria
\AND
 \name Magdalena Ortiz
 \email ortiz@kr.tuwien.ac.at \\
 \addr TU Wien, Austria
 \AND
 \name Mantas \v{S}imkus
 \email simkus@dbai.tuwien.ac.at \\
 \addr TU Wien, Austria
}

\maketitle

\begin{abstract}
  In many scenarios, complete and incomplete information coexist. For
  this reason, the knowledge representation and database communities
  have long shown interest in simultaneously supporting the closed-
  and the open-world views when reasoning about logic theories.  Here
  we consider the setting of querying possibly incomplete data using
  logic theories,
  formalized as the evaluation of an \emph{ontology-mediated query
    (OMQ)} that pairs a query with a theory, sometimes called an
  \emph{ontology}, expressing background knowledge.
  This can be further enriched by specifying a set of \emph{closed
    predicates} from the theory that are to be interpreted under the
  closed-world assumption, while the rest are interpreted with the
  open-world view.  In this way we can retrieve more precise answers
  to queries by leveraging the partial completeness of the data.

  The central goal of this paper is to understand the relative
  expressiveness of ontology-mediated query languages in which
  the ontology part is written in the expressive
  Description Logic (DL) $\ALCHOI$ and includes 
  a set of \emph{closed} predicates. We consider a restricted class of conjunctive queries.  Our
  main result is to show that every query in this \emph{non-monotonic}
  query language can be translated \emph{in polynomial time} into
  \datalog with negation as failure under the stable model semantics.
  To overcome the challenge that \datalog has no direct means to
  express the existential quantification present in $\ALCHOI$, we
  define a two-player game that characterizes the satisfaction of the
  ontology, and design a \datalog query that can decide the existence
  of a winning strategy for the game.  If there are no closed
  predicates---in the case of querying an $\ALCHOI$ knowledge base---our translation yields a positive disjunctive \datalog program of
  polynomial size.  To the best of our knowledge, unlike previous
  translations for related fragments with expressive (non-Horn) DLs,
  these are the first \emph{polynomial time} translations.
\end{abstract}

\section{Introduction}\label{sec:introduction}
\noindent

\emph{Ontology-mediated queries} (OMQs) are an extension of standard
database query languages, which allow to obtain more answers from
incomplete data by using domain knowledge given as an \emph{ontology}.
In a nutshell, an ontology is a logical theory describing a domain of
interest. It explicates the entities that are relevant in the domain,
their properties, and the relations between them, in a formal logical
language with well-defined syntax.  The most popular languages for
writing ontologies are based on \emph{Description Logics} (DLs), a
prominent family of decidable fragments of first-order logic
\cite{dlhandbook}, and on related rule languages like \datalog${\pm}$
\cite{DBLP:conf/pods/CaliGL09}.  Advocated mostly as a tool to provide
users and applications a common understanding of the domain of
discourse, ontologies are applied in a diversity of areas, including
the Web, life sciences, and many others.

A successful use case for ontologies is in data management, where the
\emph{ontology-based data access (OBDA)} paradigm advocates the use of
ontologies as conceptual views of possibly incomplete and
heterogeneous data sources, facilitating their access.  OMQs were
proposed in this setting. 
In a nutshell, an 
OMQ pairs a database query and an ontology, and the latter is used to
infer better answers from incomplete data.  \newt{For example, the
  following axiom expresses that Bachelor students are students.
  It is written in DL syntax, as a so-called \emph{TBox} $\T_1$:
  \begin{equation}
    \mathsf{BScStud} \ISA \mathsf{Student} 
  \end{equation}
  Consider an OMQ $(\T_1,q_1)$ that pairs $\T_1$ and a (unary)
  \emph{instance query} $q_1(x) = \mathsf{Student}(x)$ that retrieves
  all students.  
  If we pose this OMQ over a database that stores that Alice is a
  Bachelor student and Bob a student; in our case, written as a DL
  \emph{ABox} with two assertions
  $\{\mathsf{BScStud}(a),\mathsf{Student}(b)\}$, then we obtain both
  $a$ (Alice) and $b$ (Bob) as \emph{(certain) answers}.  Note that
  $q_1$ alone, without the knowledge in $\T_1$, would only retrieve
  $b$ as an answer.} 
OMQs have received extensive attention over the last decade, and are
still motivating major research efforts in the database and knowledge
representation research communities; see, for example, the following
works and their references
\cite{DBLP:journals/tods/BienvenuCLW14,DBLP:conf/rweb/BienvenuO15,DBLP:conf/ijcai/GottlobMP15}.

The \emph{open-world semantics} is the central property that makes
OMQs suitable for incomplete data sources.
The facts in the database only describe the world partially; other
facts that are not present may also be true.  Here Alice is a
student, 
although this is not present in the data.
However, viewing \emph{all} data as incomplete can result in too few
certain answers.  For example, we let the TBox $\T$ contain $\T_1$ and
the following axioms, stating that every student attends some
course, and that bachelor students cannot attend graduate courses:
%
\[ \mathsf{Student} \ISA \Some{\mathsf{attends}}{\mathsf{Course}}
  \qquad \mathsf{BScStud} \ISA \All{\mathsf{attends}}{\neg
    \mathsf{GradCourse}} \]
Now we take the following ABox $\A$:
%
\[ \{\mathsf{BScStud}(a), \quad \mathsf{Course}(c_1), \qquad
  \mathsf{Course}(c_2), \qquad \mathsf{GradCourse}(c_2) \}
\]
There are no certain answers for the OMQ $(\T,q_2)$ with
$q_2(x,y) =\mathsf{attends}(x,y)$; 
intuitively, we do not know which course Alice attends.  However, if
$c_1$ and $c_2$ are known to be the \emph{only} courses,
then
$(a,c_1)$ should become a certain answer, since Alice must attend some
course and it cannot be $c_2$. 

Reasoning in the presence of partial completeness has always been
considered a very relevant but 
challenging problem, both in the databases and in the knowledge
representation communities.  The last few years have seen a renewed
interest in the topic.
For a survey of many works on databases combining complete and
incomplete information, the reader may refer to
\cite{DBLP:conf/gvd/RazniewskiN14} and its references. 
In DLs, \emph{closed predicates} have been advocated as a natural way
to combine complete and incomplete knowledge. Here a set of predicates
is explicitly declared as closed; these are assumed complete and given
a \emph{closed-world}
semantics~\cite{DBLP:journals/entcs/FranconiIS11,DBLP:conf/ijcai/LutzSW13}.
In our example above,
$(a,c_1)$ is a certain answer if we declare
$\mathsf{Course}$ as \emph{closed predicate}. Note that closed
predicates make queries \emph{non-monotonic}: although
$(a,c_1)$ is a certain answer over
$\A$, it is not a certain answer over the extended set of facts $\A' =
\A \cup \{\mathsf{Course}(c_3)\}$.

The main goal of this paper is to investigate the relative
expressiveness of OMQ languages where the ontology is written in an
expressive DL with closed predicates, in terms of more traditional
query languages like \datalog.  More precisely, we are interested in
the following problem: given an OMQ $Q$, is it possible to obtain a
query $Q'$ in a suitable fragment of \datalog such that, for any ABox
$\A$, the certain answers to $Q$ and $Q'$ coincide,
and the size of $Q'$ is bounded by a polynomial in $Q$?  Such
rewritings are central in OMQ research.
The existence of $Q'$ and its size are crucial for understanding the
expressive power and succinctness of different families of OMQs.
For some less expressive DLs, 
rewritings have also paved the way
to 
scalable OMQ answering by reusing existing technologies.
As we discuss below, finding such rewritings has been a very active
research field and produced many results. However, the vast majority
of these results are for OMQs whose ontological component is in a
so-called \emph{lightweight DL}, and use as target languages either
first-order queries, or fragments of standard (positive) \datalog,
which have as common feature polynomial data complexity.  In contrast,
we consider the expressive DL $\ALCHOI$, which can express
significantly more complex structures of knowledge, but as it is
co\np-hard in data complexity \cite{Schaerf1994b}, a rewriting into
plain (non-disjunctive) \datalog cannot exist under standard
computational complexity assumptions.
Moreover, 
for DLs that are not polynomial in combined complexity, most
rewritings need exponential time and generate an exponentially larger
\datalog program.  Last but not least, the non-monotonicity caused by
\emph{closed predicates} means that our OMQs cannot be rewritten into
monotonic variants of \datalog, like positive \datalog (with or
without disjunction).

The OMQs in this paper take the form $(\T,\Sigma,\q)$, where $\T$ a
TBox in the very expressive DL $\ALCHOI$ and $\Sigma$ are the closed
predicates. For the query language of $\q$, a desirable candidate is
the class of conjunctive queries (CQs), a popular database query
language and the language of choice in the OMQ literature.
Unfortunately, 
allowing CQs and closed predicates precludes the existence of a
polynomial sized $Q'$ under standard complexity assumptions, even if
we restrict the ontology language.  Indeed, in the presence of closed
predicates, OMQ answering with CQs is \twoexptime-hard already for the
very restricted fragments of $\ALCHOI$ known as $\mathcal{EL}$ and
DL-Lite$_\mathcal{R}$~\cite{DBLP:conf/kr/NgoOS16}, while entailment in
\datalog with negation is in co\nexptimenp.  Therefore in our OMQs,
$\q$ is from a restricted class of CQs that we call \emph{c-acyclic},
and that generalizes instance queries and acyclic conjunctive
queries.
We 
propose a \emph{polynomial time translation} of any given OMQ
$Q = (\T,\Sigma,\q)$ as above, into a \datalog program extended with
negation under the stable model semantics. 

More precisely, the paper presents two main 
contributions:
\begin{itemize}
\item We first provide a game-theoretic characterization of the
    semantics of our OMQs.  To overcome the challenge that \datalog
    has no direct means to express the existential quantification
    present in $\ALCHOI$, we define a two-player game that, in a
    nutshell, verifies whether a given structure (essentially, an
    input ABox $\A$ possibly extended with additional facts about the
    same individuals) can be extended into a model of the given
    ontology, while preserving the (non)entailment of ground facts.




\item We construct, for a given c-acyclic OMQ $Q= (\ALCHOI,\Sigma, \q)$ a query $(P_Q,q)$, where $P_Q$ is a \datalog program with negation, which verifies the existence of a winning strategy for the game above.  
We prove that the certain answers of $Q$ and $(P_Q,q)$ coincide over every input ABox $\A$ over the concept and role names occurring in $\T$. $P_Q$ does not depend on the input data, and crucially, its size is polynomial. The rules in $P_Q$ are non-ground, and use predicates whose arities depend on the input $\T$. 
To our knowledge,
this is the first polynomial time translation of an expressive
(non-Horn) DL into a variant of \datalog.
\end{itemize}

With these constructions in place, we can also obtain the following results:

\begin{itemize}
\item  We obtain optimal complexity bounds, for deciding whether a
  tuple $\vec{a}$ is a certain answer of $(\T, \Sigma, \q)$ over a
  input ABox $\A$, namely co\np-completeness in data complexity and
  \exptime-completeness in combined complexity. The upper bound for
  data complexity follows from the complexity of reasoning in \datalog
  with negation. For the upper bound on combined complexity, we
  carefully analyze the shape of the program $P$ that results from our
  rewriting, and we show that the resulting rules fall into a suitably
  restricted fragment. Note that entailment from unrestricted \datalog
  programs with negation is co\nexptime-complete
  \cite{DBLP:journals/tods/EiterGM97}.

\item We study the case when closed predicates are absent. We show
  that the considered OMQs can be translated into a positive
  disjunctive \datalog program that uses, and inherently needs, the
  built-in inequality predicate $\neq$. If in addition nominals are
  disallowed from ontologies, then the inequality predicate is not
  used.

\item Our polynomial time translations heavily rely on a pair of
  constants, intuitively corresponding to the bits $0$ and $1$, that are
  introduced by rules of the target program, i.e., our target variants of \datalog support
  constants in general rules. We argue that
  disallowing this leads to the non-existence of polynomial time
  translations, under the usual assumptions in complexity theory.

\end{itemize}
A preliminary version of this article has been published in
\cite{DBLP:conf/ijcai/AhmetajOS16}, and a simplified version of this translation for $\mathcal{ALCHI}$ can be found in \cite{version-without-closedpredicates}. 

\subsection{Related Work}

\paragraph{Combining complete and incomplete information}
In the database community,  there is considerable amount of work on viewing
the data as a mixture of complete and incomplete information. 
For example, in Master Data
Management~\cite{Fan:2010:RIC:1862919.1862924}, often some
relations in an organization's database can be considered complete 
(e.g., the table of company's products),
while the remaining ones may miss tuples, i.e. may be incomplete (e.g., 
the 
customers' phone numbers). 
The question of whether a database
  with complete and incomplete predicates has complete information to
  answer a query has received significant
  attention~\cite{Fan:2010:RIC:1862919.1862924,DBLP:journals/tods/DengFG16}; 
these works focus on the complexity for variations of the problem 
  when inclusion
  dependencies are given as database constraints.

When reasoning about privacy, some or all tuples in
selected database relations may be hidden for privacy reasons, and thus from
the users' viewpoint such relations are 
incomplete~\cite{DBLP:conf/lics/BenediktBCP16, DBLP:conf/aaai/BenediktGK17}. 
This work, quite related to ours, focuses on  \emph{visible} and \emph{invisible} tables. The visible tables are
  database predicates that are assumed complete, while the invisible
  tables are predicates that should be seen as incomplete. Among
  others, the authors consider the Positive Query Implication problem
  for guarded \emph{tuple-generating dependencies (TGDs)} (also known as
  \emph{existential rules}) and guarded disjunctive TGDs, which matches the usual query answering 
  problem for DLs in the presence of closed predicates, the task we consider in this paper. The DL $\ALCHOI$ can be
  seen as an ontology language that is orthogonal to guarded disjunctive TGDs,
  which support predicates of arbitrary arity.  However, unlike us, they focus on the
  complexity of answering variations of conjunctive queries, showing
  that query answering is already \exptime-complete in data
  complexity, and 2\exptime-complete in combined complexity. 
We consider only a restricted class of conjunctive queries. 

In the DL setting, closed predicates were introduced in
  \cite{SeFB09}. In that work, ABoxes are replaced by \emph{DBoxes}, which are
  syntactically the same, but interpreted by taking all
  predicates occurring in the ABox as closed. It was shown there that this
  change   results in co\np hard instance checking, even for lightweight DLs.  
The authors also prove
  that in expressive DLs with
  nominals DBoxes can be expressed  rather naturally as part of the
    TBox. 
We remark that this encoding is only for a given DBox (i.e., dataset),  
 and hence it is not useful for obtaining data-independent rewritings of OMQs
 with closed predicates. 
More recent works on OMQs with closed predicates have focused on the complexity
of their evaluation, e.g.,
\cite{DBLP:conf/kr/NgoOS16,DBLP:conf/ijcai/LutzSW13,DBLP:journals/entcs/FranconiIS11}. 

Another research direction that brings some form of closed-word
  reasoning to DLs is the so-called \emph{hybrid knowledge bases} that
  combine DL ontologies with \datalog rules (see, e.g.,   \cite{DBLP:journals/tkde/Lukasiewicz10,DBLP:conf/kr/Rosati06,DBLP:journals/ai/EiterILST08,DBLP:journals/jacm/MotikR10}).
  Hybrid knowledge bases usually consist of a logic program
  $\curlyP$ and a TBox $\T$ in some DL.  Typically, the signature is
  partitioned into DL and rule predicates, and the closed world
  assumption is applied to the latter, but not to the former.
  Most hybrid formalisms impose syntactic \emph{safety} conditions, 
  necessary for reasoning to be decidable. 
  Some of these formalisms allow for default negation, and can thus be seen as extensions of DLs 
  with non-monotonic reasoning. However, most hybrid languages 
  cannot naturally simulate the full
  effect of closed predicates in DLs. 
A relevant exception is the recently introduced language of 
 \emph{Clopen
    knowledge bases} \cite{DBLP:conf/aaai/BajraktariOS18} that combines non-monotonic disjunctive \datalog
  rules under the stable model semantics 
   with DL ontologies, as
  well as open and closed predicates. For reasoning about such KB, an
  algorithm that rewrites a given KB into non-monotonic disjunctive
  \datalog was proposed and implemented.  However, the rewriting
    algorithm is only for a fragment that significantly restricts the
    use of closed predicates, and results in a program of exponential
    size.

\paragraph{Query rewriting}

In the context of DLs, rewriting OMQs into standard query languages, such as
\emph{first-order (FO) queries} (SQL, non-recursive \datalog)  or \datalog, is
considered one of the most prominent approaches for OMQ answering.
FO-rewritings for members of the \emph{DL-Lite} family were originally
proposed in ~\cite{DBLP:journals/jar/CalvaneseGLLR07,
DBLP:journals/ai/CalvaneseGLLR13}, but those rewritings
take exponential time. For \emph{DL-Lite$_{\mathcal{R}}$}, the authors
in \cite{DBLP:conf/kr/GottlobS12} propose a rewriting into a
polynomially-sized non-recursive \datalog program, assuming that the data
contains some fixed number of constants. 
 It was then shown in  \cite{DBLP:journals/ai/GottlobKKPSZ14} that without the additional assumption on the fixed number of constants polynomial FO-rewritings for OMQs consisting of DL-Lite$_{\mathcal{R}}$ and (U)CQs cannot exist. 
The rewritings presented in this paper also use a few constants, which are
supported by the \datalog variants we employ as target query languages.  
In \cite{DBLP:conf/ijcai/KontchakovLTWZ11}, the authors introduced the
\emph{combined approach} as a means to  obtain FO-rewritings for 
languages, like $\mathcal{EL}$, that are more expressive than DL-Lite. The
central idea is to, additionally to query rewriting, 
\emph{modify the ABox by incorporating inferences from the TBox}.  
\newt{Note that, in this case, the resulting rewriting is computed
  specifically for the
  given ABox; in contrast, our rewriting is independent of a particular ABox. } 


In the presence of closed predicates, the only rewritability results are
FO-rewritability for the core fragment of DL-Lite
\cite{DBLP:conf/ijcai/LutzSW15},  and a rewriting algorithm for queries that
satisfy some strong \emph{definability} criteria  \cite{SeFB09}. 
Since closed predicates cause co\np-hardness in  data
complexity already for instance queries in many lightweight DLs, 
the existence of FO-rewritings is ruled out. 
The recent rewriting for Clopen knowledge bases \cite{DBLP:conf/aaai/BajraktariOS18}, as
mentioned, imposes restrictions on the use of closed predicates, and 
results in a program of exponential size. 


Many DLs are not FO-rewritable,
but can be rewritten into monotonic \datalog queries, leading to
implemented systems, e.g.,
\cite{DBLP:journals/japll/Perez-UrbinaMH10,DBLP:conf/aaai/EiterOSTX12,RAPID}.
The pioneering work in~\cite{DBLP:journals/jar/HustadtMS07} showed that
instance queries in an expressive extension of \ALC can be rewritten
into a program in disjunctive \datalog, using  
a constant number of variables per rule, but exponentially many
rules. The first translation from \emph{conjunctive queries (CQs)} in
expressive DLs without closed predicates ($\mathcal{SH}$, $\mathcal{SHQ}$) to programs in disjunctive \datalog was introduced in \cite{DBLP:journals/jcss/EiterOS12}, but the program may contain double exponentially many predicates. For \ALC and for union of CQs, the existence of exponential rewritings into disjunctive \datalog was shown recently~\cite{DBLP:journals/tods/BienvenuCLW14}, and for restricted
fragments of $\mathcal{SHI}$ and classes of CQs translations to
\datalog were investigated
in~\cite{DBLP:journals/ai/KaminskiNG16}. 
A polynomial
time \datalog translation of instance queries was proposed in \cite{DBLP:conf/kr/OrtizRS10},  but for a so-called \emph{Horn-DL} that lacks disjunction. To  our knowledge, this was  until now
the only polynomial 
rewriting for a DL that is not FO-rewritable. 

  There is also a noticeable amount of works on query
  rewritings for variants of guarded TGDs. For instance, it was shown in ~\cite{DBLP:conf/mfcs/BaranyBC13,DBLP:conf/pods/GottlobRS14} that   OMQs, where the query is given as a (union of) conjunctive queries and the ontology as a set of guarded TGDs or more general classes of dependencies, can be rewritten into a
  plain  \datalog program. The authors in   ~\cite{DBLP:journals/tods/BienvenuCLW14} propose a rewriting into a disjunctive \datalog program for guarded \emph{disjunctive} TGDs and a union of
  conjunctive queries. However, the rewritings
  mentioned in \cite{DBLP:conf/mfcs/BaranyBC13,DBLP:conf/pods/GottlobRS14,DBLP:journals/tods/BienvenuCLW14}
  take exponential time, even if the number of variables in each (disjunctive) TGD is bounded by a constant. For guarded TGDs with a bound on the size of the schema and for linear TGDs, the authors in  \cite{DBLP:conf/kr/GottlobMP14, DBLP:conf/ijcai/GottlobMP15} propose rewritings into polynomially-sized non-recursive \datalog programs.  We note that, for linear TGDs of bounded arity, this result follows
from~\cite{DBLP:conf/kr/GottlobS12,DBLP:journals/ai/GottlobKKPSZ14}. 

Adapting the techniques presented in this paper,  we showed in
\cite{DBLP:conf/icdt/AhmetajOS18} that instance queries mediated by ontologies specified as
guarded disjunctive TGDs without closed predicates can be rewritten 
into disjunctive \datalog. 
The rewriting is polynomial if the number of variables in each TGD is bounded; 
we note that the TGDs considered in that work do not allow for constants
(essentially,   \emph{nominals}). Although 
the latter rewriting is inspired by the techniques presented in this paper, the higher arities and
  the rather relaxed syntax of guarded disjunctive TGDs makes the adaptation
  highly non-trivial. In the absence of disjunction, that is for
  non-disjunctive TGDs, we additionally propose a 
rewriting into a   (plain) \datalog program;
similarly as above, it is polynomial if the number 
of variables per TGD is bounded. 
We remark that the data complexity of such OMQs in
  the presence of closed predicates has been shown to be
  \pspace-hard~\cite{benedikt2018pspace}.  
 Therefore, under common assumptions in complexity theory, our rewritings
 cannot be generalized to such OMQs with closed predicates. 

\subsection{Organization}
This paper is organized as follows. In Section \ref{sec:prelims} we
introduce the DL $\ALCHOI$ and the normal form we consider in this
paper. We then define in Section \ref{sec:omq} the (ontology-mediated)
query answering problem for $\ALCHOI$ knowledge bases in the presence
of closed predicates. Section \ref{sec:game} is devoted to the
game-like characterization of models, and in particular, we tie the
inclusion of a tuple in a certain answer of a given OMQ over an input
ABox to the existence of a winning strategy in a two-player game. The
existence of such a strategy can be decided by a \emph{type-marking}
algorithm, which is presented in
Section~\ref{sec:marking-algorithm}. In Section \ref{sec:rewriting},
we first introduce the required variants of \datalog, and give an
implementation of the marking algorithm from Section
\ref{sec:marking-algorithm} as a \datalog program with negation under
the stable model semantics.  We proceed by carefully analyzing the
complexity of the program that results from our rewriting. We also
consider the restricted case when closed predicates, and possibly
nominals, are absent from OMQs. At the end of the section, we discuss
the important role that constants are playing in our translations.
Conclusions and directions for future work are given in Section
\ref{sec:conc}. For better readability, a part of a rather  long and technical proof is given in the
Appendix \ref{sec:appendix}.

\section{DL Preliminaries}
 \label{sec:prelims}
 
In this section we define the Description Logics and OMQ 
languages we use in this paper. 
 

\subsection{$\ALCHOI$ with Closed Predicates} 

\paragraph{Syntax} We recall the standard definition of the DL $\ALCHOI$ summarized in Table \ref{tab:syntax}. 
We assume countably infinite, mutually disjoint sets $\rolenames$ of \emph{role names},
$\conceptnames$ of \emph{concept names}, and $\indivnames$ of
\emph{individual names} to build complex \emph{concepts} and \emph{roles}. Intuitively, concepts are unary relations used to describe \emph{classes} of objects and roles are binary relations used to describe relations between objects. These expressions are then used
in \emph{inclusions} and \emph{assertions}. The former express general
dependencies to be satisfied by data instances, while the latter assert the
membership of specific (pairs of) individuals in concepts and roles. 
A \emph{(plain) knowledge base (KB)} is a tuple $\K = (\T, \A)$, where $\T$ is a finite set of inclusions called a \emph{TBox}, 
 and $\A$ is a finite set of assertions called an \emph{ABox}. 
A
 \emph{KB with closed predicates} is a triple $\K = (\T,\Sigma, \A)$, where
 $(\T, \A)$ is a plain KB and $\Sigma \subseteq \conceptnames\cup \rolenames$ is the set of \emph{closed predicates} in $\K$. 


\begin{table}
{
   \centering\small\renewcommand{\arraystretch}{1.3}
  $\begin{array}{rll@{}}
       \toprule
\multicolumn{2}{l}{\text{Roles $r$ and concepts $C_{(i)}$:}}\\
r  \longrightarrow & p ~\mid~ p^-&\\ [1mm]
  C  \longrightarrow & A ~\mid~ \top ~\mid~ \bot ~\mid~ \{a\}  ~\mid~  C_1 \AND C_2 ~\mid~ C_1 \OR C_2 ~\mid~ \NOT C  \mid\\ & \SOME{r}.{C} \mid \ALL{r}.{C} \\
 \multicolumn{2}{l}{\text{where $p\in \rolenames$,
 $A\in \conceptnames$, and $a \in \indivnames$.}}\\
    \midrule
\multicolumn{2}{l}{\text{Assertions and inclusions:}}\\
  C_1 \ISA C_2 & \text{concept inclusion}\\\
  r_1 \ISA r_2 & \text{role inclusion}\\
  A(a)  & \text{concept assertion}\\
  p(a_1,a_2)& \text{role assertion}\\
 \multicolumn{2}{l}{\text{where $C_1$, $C_2$ are concepts, $r_1$, $r_2$ are roles, $A \in \conceptnames$, $p \in \rolenames$, and $\{a, a_1,a_2\}\subseteq \indivnames$.}}\\
   \midrule
\end{array}$ 
}
  \caption{Syntax of $\ALCHOI$} 
\label{tab:syntax}

\end{table}

\medskip 

\paragraph{Semantics} We recall the usual semantics of DL KBs, 
defined in terms of \emph{interpretations}, 
relational structures over a (possibly infinite) non-empty domain
and a signature consisting of unary predicates (the concept names), binary
predicates (the role names), and constants (the individuals). 

\begin{definition}\label{def:semanticsALCHOIQ}
  An \emph{interpretation} is a pair $\I=\tuple{\dom{\I},\Int{\I}{\cdot}}$ where
  $\dom{\I} \neq \emptyset$ is the \emph{domain}, and $\cdot^{\I}$
is an \emph{interpretation function} with 
 $\Int{\I}{A} \subseteq \dom{\I}$ for each $A\in\conceptnames$,
  $\Int{\I}{p}\subseteq\dom{\I}\times\dom{\I}$ for each $p\in\rolenames$, and
  $\Int{\I}{a}\in \dom{\I}$ for each $a\in \indivnames$.  The function
  $\Int{\I}{\cdot}$ is extended to all \ALCHOI concepts and roles as usual,
  see Table~\ref{tab:semantics}.

  Consider an interpretation $\I$.  
  For an inclusion $q_1 \ISA q_2$, we say that $\I$ 
  \emph{satisfies} $q_1 \ISA q_2$ and write $\I \models q_1 \ISA q_2$ 
 if
$q_1^{\I}\subseteq q_2^{\I}$.
  For an assertion $\beta$ of the form $q(\vec{a})$, we say that $\I$
  \emph{satisfies} $q(\vec{a})$, in symbols $\I
  \models q(\vec{a})$,  
 if $(\vec{a})^{\I} \in q^{\I}$.  
  For $\Gamma$ a TBox or ABox, we write $\I\models
  \Gamma$ if $\I\models \alpha$ for  all $\alpha \in \Gamma$. 
The notion of satisfaction extends naturally to plain KBs:
$\I\models (\T,\A)$ if $\I\models \T$ and $\I\models \A$.
For KBs with closed predicates, we need the following notion.
Let $\A$ be an ABox and $\Sigma\,{\subseteq}\,\conceptnames\cup\rolenames$. We write $\I\models_{\Sigma}\A$ if: \begin{compactenum}[\it (a)] 
\item $\I\models \A$,
\item for all $A \in \Sigma \cap \conceptnames$, if $e \in A^\I$, then
  $A(e) \in \A$, and
\item for all $r \in \Sigma\cap \rolenames$, if $(e_1, e_2) \in r^\I$,
  then $r(e_1, e_2) \in \A$.
\end{compactenum} 
Then, for a KB $\K= (\T, \Sigma, \A)$, we write $\I\models \K$ if the following hold: 
\begin{compactenum}[\it(i)] 
\item $a \in \dom{\I}$ and $a^\I = a$ for each  $a \in \indivnames$ occurring in $\K$, \item $\I\models
\T$, and 
\item $\I\models_{\Sigma} \A$. 
\end{compactenum}
Note that we make the \emph{standard name assumption (SNA)} for the individuals occurring in $\K$ (condition (i)), which is common for closed predicates. 
For $\Gamma$ a plain KB or a KB with closed predicates, or a TBox $\T$, $\Gamma$ is consistent if there is some $\I$ that satisfies $\Gamma$.   For a TBox $\T$ and an inclusion $q_1 \ISA q_2$, we write $\T \models q_1 \ISA q_2$, if for every interpretation $\I$, $\I \models \T$ implies $\I \models q_1 \ISA q_2$.
In what follows, if no confusion arises, we may simply say KB to refer to a KB with closed predicates. 
   \hfill $\triangleleft$
\end{definition}

\begin{table}
  {   \centering\small\renewcommand{\arraystretch}{1.3}
   $\begin{array}{r@{~~=~~}l}
     \toprule
     \multicolumn{2}{l}{\text{Concept constructors in $\ALCHOI$}:}\\
     \top^\I & \dom{\I}\\
     \bot^\I & \emptyset\\
     \{a\}^\I & a^\I\\
     (\NOT C_1)^\I & \dom{\I} \setminus C_1^\I\\
     (C_1 \AND C_2)^\I &  C_1^\I \cap C_2^\I\\
     (C_1 \OR C_2)^\I & C_1^\I \cup C_2^\I\\
     (\SOME{r}.{C})^\I & \{e_1 \mid \text{ for some }
       e_2 \in \dom{\I}, (e_1, e_2)\in r^\I \text{ and } e_2\in C^\I\}\\
     (\ALL{r}.{C})^\I & \{e_1 \mid \text{ for all }
       e_2 \in \dom{\I}, (e_1, e_2)\in r^\I \text{ implies } e_2\in C^\I\}\\
     \midrule
     \multicolumn{2}{l}{\text{Role constructors in $\ALCHOI$}:}\\
     {(r^-)}^\I & \{(e_1, e_2) \mid (e_2, e_1 )\in r^\I\}\\   
     \bottomrule
   \end{array}$
  }
  \caption{Semantics of $\ALCHOI$ concepts and roles}
\label{tab:semantics}

\end{table}
We let $\conceptnames^+=\conceptnames \cup
\{\{a\}\mid a\in\indivnames\}\cup\{\top,\bot\}$ be the \emph{basic concepts} and the roles $r$ of the form $p$, $p^-$ with $p\in \rolenames$ be the \emph{basic roles}.
For $\Gamma$ a TBox, ABox, or KB, we denote by $\indivnames(\Gamma)$,
$\rolenames(\Gamma)$,  $\conceptnames(\Gamma)$, $\conceptnames^+(\Gamma)$, and
$\conceptnames^+(\T)$, the set of individuals, role names, concept names, and
basic concepts that occur in $\Gamma$, respectively.  
We write $r^-$ to mean $p^-$ if $r = p$ is a role name, and $p$ if $r = p^-$
is the inverse of a role name. 
We also write $r^{(-)} \in \Sigma$ to mean $r \in \Sigma$ or 
$r^- \in \Sigma$.

As usual, we use $\suprr$  to denote 
 the smallest relation that satisfies: 
\begin{inparaenum}[(1)]
	\item\label{it:sup1}  $r \suprr r$ for every basic role $r$ in $\T$, 
\item\label{it:sup4} if $r \suprr s$, then $r^- \suprr s^-$, 
\item\label{it:sup2}  if $r \suprr r_1$ and $r_1 \ISA s \in \T$, then   $r
  \suprr s$. 
\end{inparaenum} 

\subsection{Normal Form}
Our results apply to arbitrary TBoxes, but 
to simplify presentation, we consider TBoxes with a restricted syntactic structure:

\begin{definition}  A TBox $\T$ is in \emph{normal form} if each inclusion in
  $\T$ has one of the following forms: 
\begin{align*}
 &  \mathbf{(N1)}~~ B_1\AND \cdots \AND B_n \ISA B_{n+1}\OR\cdots \OR B_{k} \\[1pt]
& \mathbf{(N2)}~~ A\ISA \Some{r}{A'} \qquad \mathbf{(N3)}~~  A\ISA \All{r}{A'} \qquad \mathbf{(N4)}~~  r \ISA s
\end{align*}
where $r$ and $ s$ are basic roles,
$\{B_1, \ldots,B_k\}\subseteq \conceptnames^+ $, and $\{A, A'\} \subseteq
\conceptnames$. 
\end{definition}
 
The inclusions of type (N2) are often called \emph{existential inclusions} in this paper. 
It is a standard result that, in plain KBs, the TBox can be transformed into this normal form in polynomial time, while preserving the answers to any query that uses symbols from the input KB only. 
This does not change in the presence of closed predicates. 

\begin{proposition}  
Consider a KB with closed predicates $\K = (\T,\Sigma,\A)$. 
Then $\T$ can be transformed in polynomial time into a TBox 
$\hat{\T}$ in normal form in such a way that, for every interpretation $\I$, the following hold:
\begin{enumerate}[(1)]
\item If $\I \models (\hat{\T},\Sigma,\A)$, then $\I\models (\T,\Sigma,\A)$. 
\item If $\I \models (\T,\Sigma,\A)$, then $\I$ can be extended into $\I'$
  such that \begin{compactenum}[(a)] 
  \item $\I' \models (\hat{\T},\Sigma,\A)$, and 
\item $q^{\I} = q^{\I'}$ for all symbols $q$ in $\indivnames(\K) \cup \rolenames(\K) \cup
  \conceptnames(\K)$.

    \end{compactenum} 
\end{enumerate}
\end{proposition}

\begin{proof}[Proof (Sketch).]
The proof relies on the availability of fresh concept and role names, which are
 introduced in the place of complex expressions. 
 One can show the following:
 \begin{compactenum}[(i)] 
 \item Every model  of $\hat{\T}$ is a model of $\T$. 
 \item If $\I$ is a model of $\T$, we can obtain an $\I'$ with $\I' \models
   \hat{\T}$ by suitably
   interpreting the fresh  concept names, while  preserving 
   $q^{\I} = q^{\I'}$ for all symbols $q$ in $\indivnames(\K) \cup \rolenames(\K) \cup
  \conceptnames(\K)$. 
 \end{compactenum} 
We omit the detailed proof of (i) and (ii), since it is standard 
 (e.g., see the proof for $\mathcal{ALCHI}$ in
 \cite{DBLP:conf/ijcai/SimancikKH11},  accommodating nominals is
 straightforward).\footnote{The normal form in
   \cite{DBLP:conf/ijcai/SimancikKH11} allows for axioms $\exists r.A
   \sqsubseteq B$, which can be rewritten as $A \sqsubseteq \forall r^-.B$.}
Now (1) follows  from (i). 
For (2), since the interpretation in $\I'$ of individuals and of all predicates 
in $\A$ and $\Sigma$ is 
the same as in $\I$, it follows that $\I \models_\Sigma \A$ implies  $\I'
\models_\Sigma \A$. From the latter and from $\I' \models \hat{\T}$ we infer (a), while (b) is
a consequence of the fact that $\I$ and $\I'$ only differ in the freshly introduced~predicates. 
\end{proof}

\newcommand{\naf}{\mathit{not}~}
\newcommand{\ground}{\mathsf{ground}}

\section{Ontology-Mediated Queries}\label{sec:omq} 

We are now ready to introduce the  query languages that are the object of our
study. 
In this paper we  consider OMQs of the
form $Q=(\T, \Sigma, \q)$, where $\T$ is a TBox and $\Sigma
\subseteq \conceptnames \cup \rolenames$ is a set of closed predicates. 
A  natural candidate language for $\q$ is the language of \emph{conjunctive
  queries (CQs)}, which are essentially first-order formulas that use only existential quantification and conjunction.


\begin{definition} \label{def:cqs}
Let $ \varnames$ be a countably infinite set of \emph{variables} disjoint from $\conceptnames$, $\rolenames$, $\indivnames$. 
A \emph{(DL) atom} $\alpha$ is an expression of the form $A(x_1)$ or $r(x_1, x_2)$ 
 with $A \in \conceptnames$, $r \in \rolenames$ and $\{x_1, x_2\} \subseteq
 \varnames$. By $\vars(\Gamma)$ we denote the variables occurring in a set
 $\Gamma$ of atoms. If no confusion arises, a tuple
 of variables may be identified  with the set of its elements.  

A \emph{conjunctive query} $\q$  is an expression of the form 
 \[\exists \vec{y}.\alpha_1   \wedge \cdots \wedge \alpha_n \]
where each $\alpha_i$ is an atom and $ \vec{y}\subseteq \vars(\{\alpha_1, 
\cdots , \alpha_n \})$.
We may treat a CQ as a set of atoms. 
The  variables in  $\vec{x} = \vars(\{\alpha_1, 
\cdots \alpha_n \}) \setminus \vec{y}$ are called \emph{answer variables}, and
the arity of $\q$ is defined as the arity of $\vec{x}$. 
We may write $\q(\vec{x})$ for a CQ with answer variables $\vec{x}$. 
If there are no answer variables, then $\q()$ is a \emph{Boolean CQ}.  If $\q(\vec{x})$
consists of only one atom $\alpha$ and $\vec{x} = \vars(\{\alpha\})$, that is,
there are no existentially quantified variables, 
then $\q(\vec{x})$ is an \emph{instance query}. 

Let $\q(\vec{x})$ be a CQ and $\I$ an interpretation. 
For a mapping $\pi: \vars(\q) \rightarrow 
  \dom{\I}$, we write $\I \models \pi(\q(\vec{a}))$ if $\pi(\vec{x}) = \vec{a}$ and $\I
\models \pi(\alpha)$ for each atom $\alpha$ in $\q$ (where we slightly abuse
notation and apply $\pi$ to tuples  of variables and atoms). 
If $\I
\models \pi(\q(\vec{a}))$ for some $\pi$, 
we write $\I \models \q(\vec{a})$ and say that $\I$
\emph{satisfies} $\q(\vec{a})$.
For $\K$ a plain KB, or a KB with closed
predicates  
and for a tuple of constants $\vec{a}$, we say $\K$ \emph{entails} $\q(\vec{a})$, written $\K \models \q(\vec {a})$, if $\I
\models \q(\vec {a})$ for every $\I$ such that $\I\models \K$.

Let $Q=(\T, \Sigma, \q(\vec{x}))$ be an OMQ, where $ \q(\vec{x})$ is a CQ of
arity $n$. For an ABox $\A$ and $\vec{a} \in \indivnames^n$, 
 $\vec{a}$ is a \emph{certain answer} to $Q$ over $\A$ if 
 $(\T,\Sigma,\A) \models \q(\vec{a})$. 
We denote with $\cert(Q, \A)$ the set of
 certain answers of $Q$ over $\A$. Note that if $\q$ 
is a Boolean CQ, then 
   we have either $\cert(Q, \A) = \{()\}$ if $(\T,\Sigma,\A)
   \models \q()$, or $\cert(Q, \A) = \{\}$ otherwise. 
\end{definition}

The decision problem associated to answering OMQs is the following: given  an OMQ $Q$, a (possibly empty)
tuple of individuals $\vec{a}$, and an ABox $\A$, decide whether $\vec{a} \in \cert(Q, \A)$. All complexity bounds in this paper are for this
decision problem. 

We remark that individuals are not allowed in queries, but this is not a
limitation.  To simulate them in an OMQ $Q = 
  (\T,\Sigma,\q)$ we can use a 
  fresh $A_a \in \conceptnames$ and a fresh $x_a \in \varnames$ 
for each $a \in \indivnames$ that
  occurs in  $\q$. If the atom $A_a(x_a)$ is added to $\q$, and the axiom 
  $\{a\} \ISA A_a$ is added to $\T$, then each occurrence of $a$ can be
  replaced by $x_a$.

In the OMQ literature, CQs are very prominent, and 
 most research
so far has focused on such OMQs and their extensions. 
It is well known that answering such OMQs is a 2\exptime complete problem
whenever the TBox $\T$ is in any fragment of $\ALCHOI$ containing
$\mathcal{ALCI}$ \cite{DBLP:conf/cade/Lutz08}  or $\mathcal{ALCO}$
\cite{DBLP:conf/kr/NgoOS16}. 
In the presence of closed predicates, 2\exptime-hardness holds already for
TBoxes in the so-called $\mathcal{EL}$ and DL-Lite logics, both of which are
very  restricted fragments of $\mathcal{ALCHI}$. 

\begin{theorem}\label{thm-2exphard}\cite{DBLP:conf/cade/Lutz08,DBLP:conf/kr/NgoOS16}
Answering an OMQ $Q = (\T , \Sigma, \q)$, where $\q$ is a CQ is 2\exptime-hard in the following cases: 
\begin{enumerate}[(i)]
\item\label{thm2exp:1} $\Sigma = \emptyset$ and $\T$ is written in any DL that contains   $\mathcal{ALCI}$ or $\mathcal{ALCO}$. 
\item\label{thm2exp:2} $\Sigma \neq \emptyset$ and $\T$ is written in any DL that contains
  $\mathcal{EL}$ or DL-Lite$_\mathcal{R}$. 
\end{enumerate}
\end{theorem}



In this paper we want to show rewritability in polynomial time into \datalog variants
for classes of OMQs.   
But even the richest \datalog variants we consider allow to decide  
entailment 
in co\nexptimenp, so the existence of a polynomial rewriting for
the classes of OMQs mentioned in Theorem~\ref{thm-2exphard} 
would imply $\twoexptime \subseteq$ co\nexptimenp, 
contradicting usual assumptions in complexity theory. 
This rules out these classes  as potential candidates for 
polynomial rewritings; 
note that,  in particular, by item~(\ref{thm2exp:2}), we cannot have CQs for
any non-trivial 
DL if we want to have closed predicates. 
For this reason, we need to consider more restricted classes of 
OMQs.  We consider  below a relaxation of the class of \emph{acyclic CQs}, but
first we define a class of queries---all whose variables are mapped to
individuals---that we call \emph{c-safe}, and that strictly generalizes
instance queries. 
 
In what follows, we may write \rin as a shorthand for $r\suprr s$ for some $s^{(-)}\in \Sigma$, that is, $r$ is subsumed by some closed predicate in $\T$.  

\begin{definition}[c-variables, c-safe CQs] 
Consider an OMQ $Q=(\T,\Sigma, \q)$ where $\q$ is a CQ. 
We call $x$ a \emph{c-variable (in $Q$)} if
it is an answer variable of $\q$, or if at least one of the following holds:
\begin{compactitem}[-]
\item there exists some atom $r(x,y)$ or $r(y,x)$ in $\q$ such that \rin, or
\item there exists some atom $A(x)$ in $\q$ such that $A \in \Sigma$.
\end{compactitem}
We call $Q$ a \emph{c-safe query} if all its variables are c-variables. 
\end{definition}

\newcommand{\modelsind}{\models}

The two conditions in the definition above ensure that to answer a c-safe OMQ we only need to consider
mappings of the variables to the individuals that occur explicitly in the input ABox.
We also consider a more general class of queries that allows for conjunctions and existentially quantified variables provided they only participate in
cycles in a restricted way. It is a slight generalization of the usual
\emph{acyclicity} condition for CQs
\cite{Yannakakis:1981:AAD:1286831.1286840}.

 \begin{definition}[c-acyclic CQs]\label{def:cacyclic}
Consider an OMQ $Q=(\T,\Omega,\q)$, where
  $\q$ is a CQ such that each $\alpha \in \q$ is an atom with terms from $\varnames$.
The query graph (a.k.a. the Gaifman graph) $G(\q)$ of $\q$ is the undirected graph whose nodes are the variables of $\q$, and that has an edge between $x$ and $y$ if they occur together in  some atom in  $\q$. 
The connected components of $\q$ are those of  $G(\q)$, and $\q$ is \emph{connected} if $G(\q)$ is. 
We call a CQ $\q$ \emph{acyclic} if $G(\q)$ is  acyclic, and for every edge $(x,y)$ in $G(\q)$, there is exactly one atom in $ \q$ of the form $r(x, y)$ or $r(y,x)$.

We  call  $Q=(\T,\Omega,\q)$  \emph{acyclic  modulo c-variables} (c-acyclic for short) if the query $\q^-$ obtained by dropping all atoms $r(x,y)$ where both $x$ and $y$ are $c$-variables is acyclic,  
and every connected component in $G(\q^-)$ 
has \emph{one} c-variable. 
 \end{definition}

To answer  c-acyclic OMQs we reduce them to c-safe OMQs with the same
answers, using the so-called \emph{rolling up} technique,  
which essentially replaces by a complex
concept each acyclic query component that contains one c-variable.

\begin{definition}
 Let $\q$ be a connected, acyclic CQ, and let $x_0 \in \vars(\q)$. 
We denote by $T_{q,x_0}$ the tree that results from
 $G(\q)$ by taking $x_0$ as a root.  
 We inductively assign to each variable $x \in \vars(\q)$ a 
 concept $C_{\q,x}$ as follows:
\begin{itemize}
	\item[-] if $x$ is a leaf node then $C_{\q,x} = \bigsqcap_{C(x) \in \q}C$,
	\item[-] if the children of the node $x$ are $x_1, \ldots, x_n$, then 
	\[C_{\q,x} = \bigsqcap_{C(x) \in \q}C \AND  \bigsqcap_{1\leq i \leq n \atop r(x,x_i)\in \q \text{ or } r^-(x_i,x)\in \q} \exists r.C_{\q,x_i}\]
\end{itemize}
We call $C_{\q,x_0}$ the \emph{query concept of $\q$ w.r.t. $x_0$}.
\end{definition}

The concept $C_{\q,x}$ is in a fragment of $\ALCHOI$ called
$\mathcal{ELI}$, and it 
 correctly  captures the semantics
of the CQ $\q$. 

\begin{lemma}\label{lem:treecon}\cite{Horrocks99querycontainment} Let $\q$ be a connected, acyclic CQ. Let $x \in \vars(\q)$, and 
let $C_{\q,x}$ be the query concept of $\q$
  w.r.t. $x$. Then for every interpretation
  $\I=\tuple{\dom{\I},\Int{\I}{\cdot}}$ and any object $d \in \dom{\I}$, 
the following are equivalent: 
\begin{compactitem}[-]
\item
$d \in C_{\q,x}^\I$,
\item there is some $\pi$ such that $\I \models \pi(\q)$ with  $\pi(x) = d$.
\end{compactitem}
\end{lemma}

Every c-acyclic OMQ $Q$ can be transformed into a c-safe OMQ $Q'$ that contains the c-variables of $Q$ and that has the same certain answers as $Q$.

\begin{lemma}\label{l:acytogr} Consider a c-acyclic OMQ $Q= (\T, \Sigma,\q)$ and let $X$ be the set of c-variables of $Q$. Then $Q$ can be transformed in polynomial time into an OMQ  $Q' = (\T', \Sigma, \q')$, 
such that $\q'$ is a CQ with $\vars(\q') = X$ and for every ABox $\A$ over the concept and role names occurring in $Q$ the following holds: 
\[\cert(Q,\A) = \cert(Q',\A)\]
\end{lemma}
\begin{proof}[Proof (Sketch).]
The transformation of $Q$ to $Q'$ is as follows. We associate to $\q$ the set of connected components $\{T_1, \ldots, T_n\}$ of $G^-(\q)$.  Because $Q$ is c-acyclic, every $T_i$ is a connected acyclic undirected graph containing exactly one node $x_i$ that is a c-variable. 
We designate $x_i$ as the root of $T_i$, allowing us to view $T_ i$ as a tree and we denote by $C_{T_i, x_i}$ the query concept of $T_i$ w.r.t. $x_i$. 

The new OMQ $Q' = (\T', \Sigma, \q'(\vec{x}))$ is defined as follows: 
 \begin{align*} 
\T' & = \T \cup \{C_{T_i,x_i} \ISA A_{T_i}\mid T_i \in \{T_1, \ldots, T_n\}\}  \\\\
  \q'(\vec{x}) & =  \bigwedge_{ T_i\in \{T_1, \ldots, T_n\}} A_{T_i}(x_i) \wedge \bigwedge_{\{x_i,x_j\} \subseteq X, r(x_i,x_j) \in \q} r(x_i,x_j)
\end{align*} 
where each $A_{T_i}$ is a fresh concept name. 
Using Lemma \ref{lem:treecon} it can be easily verified that for any ABox
$\A$ over the concept and role names that appear in the OMQ $Q$, the queries $Q$ and $Q'$ produce the same certain answers.
\end{proof}

%

\section{Game Characterization of Countermodels}\label{sec:game}

Consider a c-safe OMQ $Q=(\T,\Sigma,\q(\vec{x}))$, where $\T$ is in
normal form, $\A$ is an ABox over the concept and role names occurring
in $\T$, and let $\K=(\T,\Sigma,\A)$.  To decide
$\K\not\modelsind \q(\vec{a})$ for some tuple of individuals $\vec{a}$
occurring in $\K$, we will use the notion of a \emph{core}, which is
essentially an interpretation whose domain is $\indivnames(\K)$
extended with a small number of additional individuals.  Cores fix how
the individuals of $\K$ participate in concepts and roles, ensuring
the non-entailment of $\q(\vec{a})$, the satisfaction of $\A$ and
$\Sigma$, and a partial satisfaction of $\T$.  
To decide whether there exists an interpretation $\I$ such that
$\I\models \K$ and $\I\not \modelsind \q(\vec{a})$, we will proceed in two steps: 
\begin{enumerate}[(1)]
\item Guess a \emph{core} $\I_c$ for $\K$ such that
  $\I_c \not\models \q(\vec{a})$.
\item Check that $\I_c$ can be extended to satisfy all axioms in $\T$.
\end{enumerate}

We start by formally defining cores. 

\begin{definition}\label{def:coreint} 
 For each individual $c$, and an  existential inclusion $\alpha$ of type (N2), we assume a special
  constant $c^{\alpha}$ that is prohibited to occur in ABoxes.  Such individuals are called \emph{fringe individuals}. 
  Assume a KB $\K=(\T,\Sigma,\A)$. We let $\mathit{cdom}(\K)$ contain
  the individuals $c,c^{\alpha}$ for all
  $c\in \indivnames(\K) $, and existential inclusions $\alpha$ in $\T$. A \emph{core} for $\K$ is an interpretation $\I_c = (\dom{\I_c},\cdot^{{\I_c}})$ such that
  
  \begin{enumerate}[(c1)]
  \item $\dom{\I_c} \subseteq \mathit{cdom}(\K)$, $\indivnames(\K)\subseteq \dom{\I_c}$, and $a^{\I_c}=a$
    for all $a \in \indivnames(\K)$,
  \item $\I_c \models_\Sigma \A$,
  \item $\I_c\models \alpha$ for each $\alpha \in \T$ such that
    \begin{enumerate}[(c3.1)]
    \item $\alpha$ is of the form $B_1\AND \cdots \AND B_n \ISA B_{n+1}\OR\cdots \OR B_{k}\in \T$,  
    \item $\alpha$ is of the form $A\ISA \All{r}{A'}$, 
    \item $\alpha$ is of the form $r\ISA s$, 
    \item  $\alpha$ is of the form $A \ISA \Some{r}{ A' }$ and \rin, or

    \end{enumerate}
        \item \label{core:tree} if $(d,d')\in r^{\I_c}$, then one of the
    following holds:
    \begin{enumerate}[-]
    \item $\{d,d'\}\subseteq \indivnames(\K)$,
    \item $d\in \indivnames(\K)$ and $d'=d^{\alpha}$ for some
      $\alpha\in \T$, or
    \item $d' \in \indivnames(\K)$ and $d =d'^{\alpha}$ for some
      $\alpha\in \T$.
       \end{enumerate}
  \item \label{core:innersat-n2} if $d \in \Delta^{\I_c}$ is not a
    fringe individual, then
    $d\in A^{\I_c}$ implies $d\in (\exists r.A')^{\I_c}$ for every $A\ISA\exists r.A'\in \T$.

  \end{enumerate}

\end{definition}
 
The conditions in the above definition can be explained as
follows. The domain of a core interpretation $\I_c$ for a KB $\K$
always contains the individuals in $\K$, and possibly some special
fringe individuals (condition (c1)). The condition (c2) tells us that
$\I_c$ must satisfy the ABox assertions in $\K$ and respect the closed
predicates. The condition (c3) requires all TBox inclusions to be
satisfied, with the possible exception of existential axioms
$A \ISA \Some{r}{ A' }$ such that \rnotin. Using the condition (c4) we
restrict the possible connection between the domain elements. In
particular, we allow a fringe individual $c^{\alpha}$ to be connected
by a role only to its ``parent'' individual $c$. Intuitively, the
condition (c5) makes sure that the original ABox individuals satisfy
\emph{all} TBox inclusions. Thus, intuitively, the only reason for
$\I_c$ to not be a model of the TBox is a fringe individual that
triggers an existential inclusion $A \ISA \Some{r}{ A' }$ with
\rnotin. The goal of the game-like characterization is to see how a
proper core interpretation can be extended to satisfy all existential
inclusions. The notion of an \emph{extension} is as expected:

\begin{definition}\label{def:extension} Consider a core  $\I_c$ for $\K=(\T,\Sigma,\A)$. An interpretation $\J$ is called an \emph{extension} of $\I_c$, if the following hold:
  \begin{enumerate}[(i)]
       \item $\Delta^{\I_c}\subseteq \Delta^{\J}$, 
  		\item for all concept names $A$, $A^{\I_c} = A^\J \cap \dom{\I_c}$, 
  		\item for all role names $p$, $p^{\I_c} = p^\J \cap (\dom{\I_c}\times \dom{\I_c})$, and
  		 \item  $q^\J = q^{\I_c}$ for all $q \in   \Sigma$.
  \end{enumerate}
\end{definition}


\begin{figure}
 \centering
\begin{tabular}{c@{\qquad\qquad}c} 
     \begin{tikzpicture}[-,scale=0.75,label distance=-1mm]
      \tikzstyle{every node}=[inner sep=0pt,minimum size=6mm];
      \tikzstyle{every path} = [draw,font=\footnotesize];
      \tikzstyle{every label}=[inner sep=0pt,font=\footnotesize];

             \node[draw,shape=circle,label={$A_1,A_4$},outer sep=0pt] (a) at (-14,0) {$a$};
             \node[draw,shape=circle,label={$A_1,A_3$},outer sep=0pt] (b) at (-11,0) {$b$};
             \node[draw,shape=circle] (c) at (-9,0) {$c$};        

             \node[draw,shape=circle,double,label=below:{$A_2,A_3$},outer sep=0pt]
                 (a1) at (-14,-2) {$a^{\alpha_1}$}; 

             \node[draw,shape=circle,double,label=below:{$A_2,A_3$},outer sep=0pt] (b1) at (-12,-2) {$b^{\alpha_1}$};
             \node[draw,shape=circle,double,label=below:{$A_2,A_3$},outer sep=0pt] (b2) at (-10,-2) {$b^{\alpha_3}$};         

   	      \tikzstyle{every node} = []; 
                   \draw (a) [left,->] to node {$r_1,r_2^-$} (a1); 
                   \draw (b) [left,->,midway,above] to node {$r_2$} (c); 
                   \draw (b) [left,->] to node {$r_1,r_2^-$} (b1); 
                   \draw (b) [left,->] to node {$r_2$} (b2); 

           \end{tikzpicture} 
& 
     \begin{tikzpicture}[-,scale=0.8,label distance=-1mm]
      \tikzstyle{every node}=[inner sep=0pt,minimum size=6mm];
      \tikzstyle{every path} = [draw,font=\footnotesize];
      \tikzstyle{every label}=[inner sep=0pt,font=\footnotesize];

           \node[draw,shape=circle,label={$A_1,A_2,A_4$},outer sep=0pt] (a) at (-14,0) {$a$};
           \node[draw,shape=circle,label={$A_1,A_3$},outer sep=0pt] (b) at (-12,0) {$b$};
           \node[draw,shape=circle,label={$A_2,A_3$},outer sep=0pt] (c) at (-10,0) {$c$};        

   	      \tikzstyle{every node} = []; 
                   \draw (a) [left,->,out=270,in=180,loop] to node {$r_1,r_2^-$} (a1); 
                   \draw (b) [left,->,out=290,in=230,midway] to node {$r_1,r_2,r_2^-$} (c); 
                   \draw (c) [left,->,out=0,in=270,loop] to node {$r_2$} (c); 
           \end{tikzpicture}
\end{tabular}
 \caption{Example cores  $\I_c^1$ and  $\I_c^2$ for the KB in Example \ref{ex:main}.} 
 \label{fig:cores}
\end{figure}
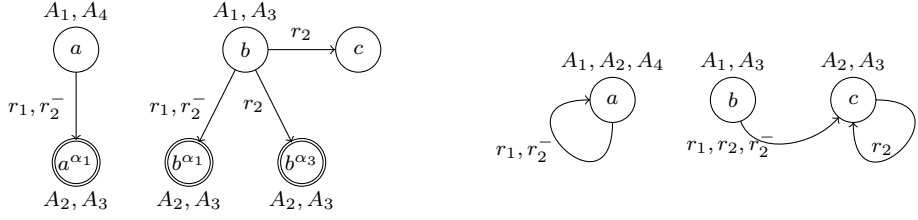

\begin{example}\label{ex:main} Consider the KB $\K=(\T,\A,\Sigma)$, where
  $\Sigma=\{A_1,A_4\}$, the TBox $\T$ contains the following 
  inclusions:

\begin{tabular}{p{4cm}@{\qquad}p{4cm}}
{\begin{align*}
\alpha_1   = A_1 & \ISA  \Some{r_1}{A_2} \\ 
\alpha_2   = 	A_2 & \ISA  A_3 \OR A_4  \\ 
\alpha_3   = 	A_3 & \ISA  \Some{r_2}{A_2}   
 \end{align*}}
&
{\begin{align*}
\alpha_4   = 	A_4 & \ISA  \All{r_2}{A_1}  \\ 
\alpha_5   =     A_3 & \ISA  \Some{r_2}{\{c\}}  \\ 
\alpha_6   =     r_1^- &\ISA r_2 
 \end{align*} }
\end{tabular}\\
and the ABox $\A$ is as follows:
\begin{align*}
  \A= & \{ A_1(a), \quad  A_4(a), \quad A_1(b), \quad A_3(b) \}
\end{align*} 
The interpretations  $\I_c^1$ and  $\I_c^2$ in Figure~\ref{fig:cores} are both cores for 
 $\I_c$. The fringe individuals are depicted with double circles.

The core $\I_c^2$ on the right is itself a model. 
The core  $\I_c^1$ is not, since the fringe individuals do not satisfy
$\alpha_3$ and $\alpha_5$, but we can extend it into a model. 
For example, we can take its infinite 
extension by adding to each fringe individual an $r_2$-child that
satisfies $A_2$ and $A_3$, which is connected via $r_2$ to $c$, and repeating the same way for all the introduced successors, see Figure \ref{fig:model}. 
The $r_2$-connections to $c$ are dashed grey and unlabeled for readability.  
\end{example}

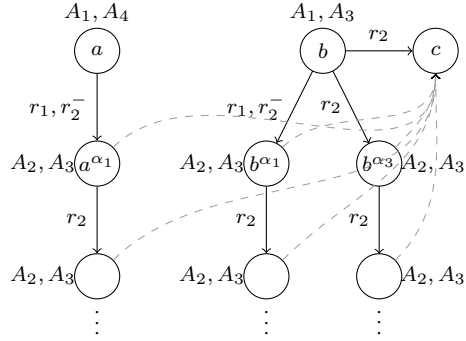
\begin{figure}
 \centering
     \begin{tikzpicture}[-,scale=0.75,label distance=-1mm]
      \tikzstyle{every node}=[inner sep=0pt,minimum size=6mm];
      \tikzstyle{every path} = [draw,font=\footnotesize];
      \tikzstyle{every label}=[inner sep=0pt,font=\footnotesize];

             \node[draw,shape=circle,label={$A_1,A_4$},outer sep=0pt] (a) at (-14,0) {$a$};
             \node[draw,shape=circle,label={$A_1,A_3$},outer sep=0pt] (b) at (-10,0) {$b$};
             \node[draw,shape=circle] (c) at (-8,0) {$c$};        

             \node[draw,shape=circle,label=left:{$A_2,A_3~$\quad}]
                 (a1) at (-14,-2) {$a^{\alpha_1}$}; 
             \node[draw,shape=circle,label=left:{$A_2,A_3~$\quad}] (b1) at (-11,-2) {$b^{\alpha_1}$};
             \node[draw,shape=circle,label=right:{~$A_2,A_3$},outer sep=0pt] (b2) at (-9,-2) {$b^{\alpha_3}$};         

             \node[draw,shape=circle,label=left:{$A_2,A_3~$\quad},label=below:{$\vdots$}]
                 (a2) at (-14,-4) {}; 
             \node[draw,shape=circle,label=left:{$A_2,A_3~$\quad},label=below:{$\vdots$}] (b12) at (-11,-4) {};
             \node[draw,shape=circle,label=right:{~$A_2,A_3$},outer sep=0pt,label=below:{$\vdots$}] (b22) at (-9,-4) {};         

   	      \tikzstyle{every node} = []; 
                   \draw (a) [left,->] to node {$r_1,r_2^-$} (a1); 
                   \draw (b) [left,->,above] to node {$r_2$} (c); 
                   \draw (b) [left,->] to node {$r_1,r_2^-$} (b1); 
                   \draw (b) [left,->] to node {$r_2$} (b2); 

                   \draw (a1) [left,->] to node {$r_2$} (a2); 
                   \draw (b1) [left,->] to node {$r_2$} (b12); 
                   \draw (b2) [left,->] to node {$r_2$} (b22); 

                   \draw (a1) [->,dashed,in=270,draw=gray!75] to node {} (c); 
            \draw (a2) [->,dashed,in=270,draw=gray!75] to node {} (c); 
            \draw (b1) [->,dashed,in=270,draw=gray!75] to node {} (c); 
            \draw (b12) [->,dashed,in=270,draw=gray!75] to node {} (c); 
            \draw (b2) [->,dashed,in=270,draw=gray!75] to node {} (c); 
            \draw (b22) [->,dashed,in=270,draw=gray!75] to node {} (c);

           \end{tikzpicture} 
 \caption{A model of $\K$ that extends the core $\I_c^1$.}
 \label{fig:model}
\end{figure}

A core and its extensions coincide on the
assertions they satisfy over the individuals occurring in the input
KB, and the non-entailment of a set of assertions from a KB is always
witnessed by some core that can be extended into a model.

\begin{lemma}\label{thm:decomp} Assume a c-safe OMQ
  $Q=(\T,\Sigma, \q(\vec{x}))$ and a KB $\K = (\T,\Sigma, \A)$.  Then
  $\K\not\modelsind \q(\vec{a})$ for some tuple of individuals
  $\vec{a}$ occurring in $ \K$ iff there exists a core $\I_c$ for $\K$
  such that
  \begin{enumerate}[\it(1)]
  \item \label{en:1} $\I_c \not\modelsind \q(\vec{a}) $, and
  \item \label{en:2} there exists an extension $\J$ of $\I_c$ such
    that $\J \models \K$.
  \end{enumerate}
\end{lemma}

\begin{proof}
  The ``$\Leftarrow$'' direction is not difficult. Observe that if
  $\J$ is an extension of $\I_c$, then $\J$ and $\I_c$ agree on the
  participation of constants from $\K$ in all concept and role
  names. I.e., for all constants $c,d\in \indivnames(\K)$ and all
  concept names $A$ and all role names $r$, we have $c\in A^{\J}$ iff
  $c\in A^{\I_c}$, and $(c,d)\in r^{\J}$ iff $(c,d)\in r^{\I_c}$. This
  observation, together with the facts that $Q$ is c-safe and
  $\J \models \K$, implies that $\J \not\modelsind q(\vec{a})$, and,
  hence, $\K\not\modelsind \q(\vec{a})$.

  
  For the ``$\Rightarrow$'' direction, let $\I$ be an interpretation
  such that $\I \models \K$ and $\I \not\modelsind \q(\vec{a})
  $. W.l.o.g.\,we assume $q^{\I} = \emptyset$ for all
  $q \not\in \conceptnames(\K) \cup \rolenames(\K)$. For every element
  $e\in \Delta^{\I}$ and every inclusion
  $\alpha= A\ISA \exists r.A'\in \T$ with $e\in A^{\I}$, let
  $wit(e,\alpha)$ denote an arbitrary but fixed element
  $e'\in \Delta^{\I}$ with $(e,e')\in r^{\I}$ and $e'\in
  (A')^{\I}$. Note that $e'$ always exists because $\I$ is a model of
  $\K$. We build now the ``unravelling'' of $\I$. For a word $w$ of
  form $c\cdot(\alpha_1,e_1)\cdots (\alpha_n,e_n)$, we let
  $\mathit{tail}(w)=c$ in case $n=0$, and $\mathit{tail}(w)=e_n$ in
  case $n>0$. Let $\mathit{Paths}$ be the $\subseteq$-minimal set of
  words that includes $\indivnames(\K)$, and satisfies the following rule: if $w \in \mathit{Paths}$, $\alpha$ is an existential inclusion
    in $\T$, $e=wit(\mathit{tail}(w),\alpha)$ is defined, and
    $e\not\in \indivnames(\K)$, then
    $w\cdot (\alpha,e)\in \mathit{Paths}$.
We unravel  $\I$ into an interpretation $\mathcal{U}$. More precisely,  $\mathcal{U}$ is defined  as follows:
  \begin{enumerate}[-]
  \item $\Delta^{\mathcal{U}}= \mathit{Paths}$, and
    $c^{\mathcal{U}} = c$ for all $c \in \indivnames(\K)$,
  \item for all concept names $A$,
    $A^{\mathcal{U}}= \{w \in \Delta^{\mathcal{U}}\mid \mathit{tail}(w)\in A^{\I}\}  $,

  \item for all role names $r$, $r^{\mathcal{U}}=\{(w,w') \in \Delta^{\mathcal{U}} \times \Delta^{\mathcal{U}} \mid (\mathit{tail}(w), \mathit{tail}(w') )\in r^{\I}\}$.
  

  \end{enumerate}
  It is standard to see that $\mathcal{U} \models \K$ and
  $\mathcal{U} \not\modelsind \q(\vec{a}) $. A core $\I_c$ can be obtained from
  $\mathcal{U}$ in two steps. First, we restrict the domain of $\mathcal{U}$ to paths of
  length at most 1, i.e.\,to domain elements of the form $c$ and
  $c\cdot (\alpha_1,v_1)$. Second, all domain
  elements $c\cdot (\alpha_1,v_1)$ are replaced by the constant
  $c^{\alpha_1}$. Note that since
  $w\cdot (\alpha,v_1)\in\mathit{Paths}$ and
  $w\cdot (\alpha,v_2)\in\mathit{Paths}$ implies $v_1=v_2$, the second
  step is simply a renaming of domain elements, i.e.\,the second step
  produces an isomorphic interpretation. The satisfaction of the
  conditions (c1-c6) by $\I_c$ is ensured because $\mathcal{U}$ is a model of
  $\K$. Since $\I_c$ is obtained by restricting $\mathcal{U}$ and $\mathcal{U} \not\modelsind \q(\vec{a}) $, we have $\I_c \not\modelsind \q(\vec{a}) $ (condition
  {(\ref{en:1})}). Furthermore, $\mathcal{U}$ is the desired extension of
  $\I_c$ that is a model of $\K$ (condition {(\ref{en:2})}).
\end{proof}

By this lemma, deciding non-entailment of a c-safe query 
amounts  to deciding  whether there is a core that does not satisfy it, and
that can be extended into a model.
Since the domain of cores is bounded, it is not hard to 
achieve this with a polynomially sized program.  However, 
verifying whether a core can be extended into a full model is hard,  
as it corresponds to 
testing consistency (of $\I_c$ viewed as an ABox) 
with respect to $\T$,
an \exptime-hard problem already for fragments of $\ALCHOI$ such as $\mathcal{ALC}$ \cite{Schild1991}.
In order  to obtain a polynomial set of rules that solves this
\exptime-hard problem, 
we first characterize it as  a game,  revealing  a simple algorithm for it that admits an elegant implementation in non-monotonic disjunctive \datalog. %
For this we use \emph{types}, which we define as follows:

\begin{definition} 
A \emph{type $\tau $(over a TBox $\T$)}  is a subset of
  $\conceptnames^+(\T)$ such that $\bot\not\in \tau $ and $\top\in
  \tau$. 
 We denote by \types the set of all types over  $\T$.

 We say that $\tau$ \emph{satisfies} an inclusion $\alpha = B_1\AND \cdots
  \AND B_n \ISA B_{n+1}\OR\cdots \OR B_k $ of type $\mathbf{(N1)}$, if $\{B_1,
  \ldots, B_n\} \subseteq \tau$ implies $\{B_{n+1}, \ldots, B_k\} \cap
  \tau \neq \emptyset$; otherwise $\tau$ violates $\alpha$. 
Let $\I$ be an interpretation.
For an element
  $e\in\Delta^{\I}$, 
  we let
  $\type(e, \I) = \{B\in \conceptnames^+(\T)\mid e\in B^\I\} $.
 A type $\tau$ is \emph{realized by $e$ (in $\I$)} if 
$\type(e,\I) =  \tau$; we say that $\tau$ is realized in $\I$ if it is
  realized by some $e\in \dom{\I}$.
\end{definition}

\subsection{The Model Building Game}

We now describe a simple game to decide whether a given core $\I_c$ can be
extended into a model of a KB $\K$. The game is played by Bob (the
builder), who wants to extend $\I_c$ into a model, and Sam (the
spoiler), who wants to spoil all Bob's attempts. Intuitively speaking,
Sam starts the game by picking a fringe individual $a$ in $\I_c$,
whose type $\type(a,\I_c)$ becomes the \emph{current type}. Next, Sam
chooses an inclusion of the form $A\ISA \Some{r}{A'}$ that is
``triggered'' by the current type. It is then the turn of Bob to
respond by selecting a type for the corresponding $r$-successor that
satisfies $A'$. The game continues for as long as Bob can respond to
the challenges of Sam, where at each further round the response by Bob is set to be the current type in the game.

We use the term \emph{c-type} to refer to a type $\tau \in \types $ such that $\tau$ contains a nominal $\{a\}$,  $\tau \cap \Sigma \neq \emptyset$, or $A\in \tau$ for some $A\ISA \Some{r}{A'}\in \T$ with  \rin. 
Note that, in any model, these types can only be realized by the interpretation of individuals in
$\indivnames$.
  
 For a TBox $\T$, a set
  $\Sigma\subseteq\conceptnames\cup\rolenames$ and a core
  $\I_c$ for $\K = (\T, \Sigma, \A)$, where $\A$ is an ABox over the concept and role names occurring in $\T$, we define the \emph{locally consistent} set $\LC(\T,\Sigma,\I_c)$ as the set of types $\tau \in \types$ 
   such that:                 
  \begin{enumerate}
  \item[\emph{(LC$_{\mathrm{N}{1}}$)}] $\tau$ satisfies all inclusions of type $ \mathbf{(N1)}$ in $\T$.
  \item[\emph{(LC$_{{\Sigma}}$)}] If $\tau$ is a c-type, then $\tau$ must be
    realized in $\I_c$. 
\end{enumerate}

The game is played using types from $\LC(\T,\Sigma,\I_c)$ only. 
Note that, since $\I_c$ is a core, a c-type  
can only be realized by an individual in  $\indivnames(\K)$ and, 
by definition, the individuals in a  core already satisfy all axioms. 
We are now ready to describe the game. 

\paragraph{\bf The model building game} 
 Consider a core $\I_c$ for $\K$.
The game on $\I_c$ starts by Sam choosing a fringe individual 
$a \in \dom{\I_c}$, and 
$\tau = \type(a,\I_c)$  is set to be the \emph{current type}. Then:
\begin{itemize}
\item[(\ding{71})] 
 Sam chooses an inclusion $A\ISA \Some{r}{A'}~\in~\T$ such that $A \in \tau$.  
If there is no such inclusion, the game is over and Bob wins. Otherwise, Bob chooses a type $\tau'\in \LC(\T,\Sigma,\I_c)$
such that: 
\begin{enumerate}[(C1)] 
\item $A'\in \tau'$, and
\item for all inclusions $A_1\ISA \All{s}{A_2} \in \T$:
  \begin{itemize}
  \item if $r \suprr s $ 
  and $A_1\in \tau$ then
    $A_2\in\tau'$, 
  \item if 
  $r^- \suprr s$ 
  and $A_1 \in \tau'$ then $A_2 \in
    \tau$.
  \end{itemize} 

\end{enumerate}
If $\tau'$ does not exist, then  Sam wins the game.  If
$\tau'$ exists and  is a c-type, then  Bob wins the game. If $\tau'$ exists but  is not a c-type, then
$\tau'$ is set to be the current type and the game continues with a
new round, i.e.\,we go back to \ding{71}.
\end{itemize}

We now illustrate the game. To avoid clutter, we omit $\top$ from all types.   

\begin{example}
Sam and Bob play on $\I_c^1$ as follows. 
Assume that Sam starts by picking the fringe individual $a^{\alpha_1}$, 
hence $\tau = \{A_2,A_3\}$ is the current type.
The cases of $b^{\alpha_1}$ and $b^{\alpha_3}$ are the same. 

From $\tau = \{A_2,A_3\}$, Sam can choose either $\alpha_3$ or $\alpha_5$  
(see Figure~\ref{fig:nonlos-str}).
In the latter case, Bob can win the game by picking the
  type $\{\{c\}\}$. 
In the former case, Bob can pick $\tau = \{A_2,A_3\}$, which contains $A_2$ and
satisfies $\LC(\T,\Sigma,\I_c^1)$.
Note that $\tau$ is Bob's only good choice: by $\alpha_2$, any type
in $\LC(\T,\Sigma,\I_c^1)$ that contains $A_2$ must contain $A_3$ or $A_4$,
but $A_4 \in \Sigma$, and there is no type containing both $A_2$ and $A_4$
 realized in $\I_c^1$. So $\tau$ is the only type in
$\LC(\T,\Sigma,\I_c^1)$ containing $A_2$. 

Since  $\tau = \{A_2,A_3\}$ is not a c-type, the game continues from this
type, and we are back to the same situation. 
Sam can pick $\alpha_3$ or $\alpha_5$. 
If he picks $\alpha_5$ Bob can win  by picking $\{\{c\}\}$. 
If he picks $\alpha_3$, Bob can again respond with $\tau = \{A_2,A_3\}$. 
The game continues for as long as Sam keeps choosing  $\alpha_3$, and Bob
never loses. 
\end{example}

\begin{figure}
 \centering
\usetikzlibrary{shapes.multipart}     
\begin{tikzpicture}[->,scale=2, every text node part/.style={align=center}]]

           \tikzstyle{every node}=[minimum size=4mm];
      \tikzstyle{every path} = [draw,font=\footnotesize];

      \node (q1) at (0,0) {
      \begin{tikzpicture}[-,scale=0.75,label distance=-3mm]
      \tikzstyle{every node}=[draw,shape=rectangle,minimum size=4mm];
      \tikzstyle{every path} = [draw,font=\footnotesize];
           \node (1) at (-12.5,0) {$\{A_2,A_3\}$};
	   \node (2) at (-14,-2) {$A_3\sqsubseteq \Some{r_2}{A_2}$};         
		   \node (3) at (-11,-2) {$A_3\sqsubseteq \Some{r_2}{\{c\}}$};    
		   \node[label=right:{\quad B wins}] (5) at (-11,-4) {$\{\{c\}\}$};    	      
           \tikzstyle{every node} = []; 
            	\draw (1) [left,->] to node {\text{S}} (2);
            	\draw (1) [right,->] to node {\text{S}} (3);
            	\draw (2) [bend left=30,->,above] to node {\text{B}~} (1);
            	\draw (3) [left,->] to node {\text{B}} (5);
           \end{tikzpicture}
           };
          
\end{tikzpicture}
 \caption{Non-losing strategy for Bob on the core $\I_c^1$. (S stands for Sam,
   B for Bob)} 
 \label{fig:nonlos-str}
\end{figure}
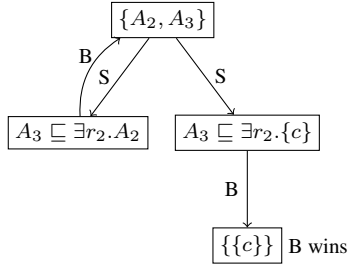  

\subsection{Runs and strategies}

We define \emph{runs} and \emph{strategies}.  Intuitively, a run
starts at a fringe individual $a$ picked initially by Sam, and
then it comprises a (possibly infinite) sequence of inclusions $\Ax_i$
picked by Sam in round $i$, and types $\tau_i$ picked by Bob in
response.  A \emph{strategy} gives a move for Bob in response to each
possible previous move of Sam, and it is called \emph{non-losing} if
it guarantees that Sam does not defeat him.
\begin{definition}
  A \emph{run} of the game on a core $\I_c$ for $\K$ is a (possibly
  infinite) sequence 
  \[ {a \Ax_1 \tau_1 \Ax_2 \tau_2 \ldots} \]
   where $i\geq 0$,  $a$ is a fringe individual in $\I_c$, each $\alpha_i$ is an
  existential inclusion in $\T$, each $\tau_i$ is a type over
  $\T$, and either there are no c-types in it, or 
${a \Ax_1 \tau_1 \Ax_2 \tau_2 \ldots \Ax_\ell \tau_\ell}$ is finite and
$\tau_\ell$ is the only  c-type.

  A \emph{strategy for Bob} is a partial function \str that maps each 
   pair of a type $\tau$ that is not a c-type  
   and an inclusion
    $A \ISA \Some{r}{A'}\in \T$ with $A \in \tau$,
to a type $\tau'$ that satisfies (C1) and (C2).
A run $a \Ax_1 \tau_1 \Ax_2 \tau_2 \ldots $ \emph{follows a
      strategy \str}, if $\tau_0=\type(a,\I_c)$, and
$\tau_{i} = \str(\tau_{i-1},\Ax_i)$ for every
    $i > 0$.

For a finite run $w$, we let $\tail(w) = \type(a,\I_c)$ if $w=a$, and 
$\tail(w) = \tau_\ell$ if $w = a \ldots \Ax_\ell \tau_\ell$ with $\ell \geq 1$. 
The strategy $\str$ is called \emph{non-losing} on $\I_c$ if
for every finite run $w$ that follows \str, $\tail(w) \in \LC(\T,\Sigma,\I_c)$
and $\str(\tail(w),A \ISA \Some{r}{A'})$  is defined
for every $\tail(w)$ that is not a c-type and every $A \ISA \Some{r}{A'} \in \T$ such that $A \in \tail(w)$. 
\end{definition}

Note that there are no runs on a core that has no fringe individuals, and therefore, every strategy is non-losing on it. 



\begin{example}\label{ex:main2}
Figure~\ref{fig:nonlos-str} depicts the following non-losing strategy \str for
Bob $\I_c^1$:
\begin{align*}
\str(\{A_2,A_3\},A_3\sqsubseteq
                          \Some{r_2}{A_2}) & {} = \{A_1,A_2,A_4\}\\  
\str(\{A_2,A_3\},A_3\sqsubseteq \Some{r_2}{\{c\}}) & {} = \{\{c\}\}. 
\end{align*}
Note that for $\I_c^2$ every strategy is winning since there are no fringe
 individuals. 
\end{example}

\subsection{Correctness of the game characterization} 

The following proposition establishes that the existence of a  non-losing
  strategy for Bob on a given core correctly characterizes the existence of a
  model of $\K$ extending that core. 


\begin{proposition}\label{p:non-losing}
 Consider a KB $\K$ and a core $\I_c$ for $\K$. There is a non-losing strategy \str for Bob on $\I_c$ iff  
$\I_c$ can be extended into a model $\J$ of $\K$.
\end{proposition}
\begin{proof}
For the ``$\Rightarrow$'' direction, from an arbitrary non-losing \str for Bob
on $\I_c$, we build an interpretation $\J$ as follows.  

For each c-type $\tau$ that is realized in $\I_c$, we let $a_\tau$ denote
a fixed, arbitrary individual realizing $\tau$; 
in particular, $a_\tau = b$ if $\{b\} \in \tau$.

Let $S$ be the  set of all finite runs that follow $\str$ on $\I_c$.
We consider two subsets  $\run$ and  $\srun$ of $S$ defined as follows: 
\begin{align*}
\run(\I_c,\str) =  & \{ a\Ax_1\tau_1\cdots \Ax_\ell\tau_\ell \in S \mid  \mbox{each $\tau_i$ for $1 \leq i \leq \ell$ is \emph{not} a c-type}  \} \\
 \srun(\I_c,\str) = & \{ a\Ax_1\tau_1\cdots\Ax_\ell\tau_\ell \in S \mid 
 \mbox{~$\tau_\ell$ is a c-type, and for each $1 \leq i < \ell$,} \\ 
    & \phantom{\{ a\Ax_1\tau_1\cdots\Ax_\ell\tau_\ell \in S \mid~} \mbox{$\tau_i$ is
 not a c-type} \} . 
\end{align*} 
Roughly, $\run$ collects all the finite runs from $S$ that do not contain any c-types and $\srun$ collects all those that end in a c-type.

\smallskip
We define the domain of the desired $\J$: 
\[ \dom{\J} = \indivnames(\K) \cup  \run(\I_c,\str)  \]
As $\cdot^{\J}$ we take an interpretation function that satisfies  the following 
 equalities, for all $a \in \indivnames(\K)$, all $A \in \conceptnames$, and all  $p \in \rolenames$: 
\begin{align*} 
a^\J & = a \\ 
A^\J &  = A^{\I_c} \cup  \{ w \in \dom{\J}  \mid A \in \textit{tail}(w) \} \\  
p^\J & = p^{\I_c} \cup \{ (w,w\Ax\tau)\in  \dom{\J} \times \dom{\J} \mid r \suprr p, \Ax = A\ISA \Some{r}{A'} \in \T \} \cup {} \\ 
  &  \phantom{{} = p^{\I_c} \cup {}}  \{ (w\Ax\tau,w)\in  \dom{\J} \times \dom{\J}   \mid  r^- \suprr p, \Ax = A\ISA \Some{r}{A'} \in \T \} \cup {} \\
 &  \phantom{{} = p^{\I_c} \cup {}}   \{ (w,a_{\tau})\in  \dom{\J} \times \dom{\J}   \mid  w{\Ax}\tau \in \srun(\I_c,\str),r \suprr p \} \cup {} \\ 
 &  \phantom{{} = p^{\I_c} \cup {}}   \{ (a_{\tau},w)\in  \dom{\J} \times \dom{\J}  \mid  w{\Ax}\tau \in \srun(\I_c,\str),   r^- \suprr p \} 
\end{align*}  
Note that taking such an interpretation is possible, since $w{\Ax}\tau \in
\srun(\I_c,\str)$ implies that $\tau$ is a  c-type in
$\LC(\T,\Sigma,\I_c)$, which guarantees that $\tau$ is realized in $\I_c$
and that $a_{\tau}$ exists.  We also note that each $w \in \dom{\J}$ is such
that $\type(w, \J) = \tail(w)$,  and, in particular, for fringe individuals from $a \in
\dom{\I_c}$, it is the case that $\type(a, \I_c) = \type(a,\J) = \tail(a)$.

It is left to prove that  $\J$ is an extension of $\I_c$, and a model of $\K$. 
For the former,  we need to show that $\J$ satisfies the following conditions:

\begin{enumerate}[(i)]
 \item $\dom{\I_c} \subseteq \dom{\J}$, 
 \item  $A^\J \cap \dom{\I_c} = A^{\I_c}$,  
 \item $r^\J \cap (\dom{\I_c} \times \dom{\I_c}) = r{^{\I_c}}$, and 
 \item  $q^\J = q^{\I_c}$ for all $q \in   \Sigma$.
\end{enumerate}
 Clearly,  $\dom{\I_c} \subseteq \dom{\J}$ since for all fringe individuals $a \in \dom{\I_c}$,
 $w=a$ is a run in $\run(\I_c,\str) = \dom{\J}$, while for the individuals occurring in $\K$ the claim holds by construction and since $a^{\I_c} = a$ by the definition of a core. 
Next, for (ii) and (iii) we observe that,  by construction of $\J$,  $p^{\I_c} \subseteq p^\J$ for all role names $p$ and $A^{\I_c} \subseteq A^\J$ for all concept names $A$. 
That no other individual or fringe individual is added to $A$ is ensured by construction of $\run(\I_c,\str)$, and the fact that  $\tail(a) = \type(a,\I_c)=\type(a,\J)$. Moreover, for any other pair $(w_1,w_2)$ added to $p^\J$,
we have that at least  one $w_i$ in the pair is of the form $a\ldots\alpha_i\tau_i$
for some $\tau_i$, with $i\geq 1$, that is not a c-type, 
and thus $w_i \not\in \dom{\I_c}$. The latter is ensured by construction of $\J$ and the fact that runs of the form  $a\alpha_1\tau_1$, where $\tau_1$ is a c-type, are not allowed in  $\srun(\I_c,\str)$. 

Finally we show that $q^\J = q^{\I_c}$ for all role names and concept names $q \in   \Sigma$. More precisely we need to show that:
 \begin{enumerate}[\it (a)] 
\item for all $A \in \Sigma \cap \conceptnames$, $e \in A^\J$ iff
  $e \in A^{\I_c}$, and
\item for all $r \in \Sigma\cap \rolenames$,  $(e_1, e_2) \in r^{\J}$ iff $r(e_1, e_2) \in r^{\I_c}$.
\end{enumerate} 
 By (ii) and (iii) we know that (a) and (b) hold for all $e$ in
 $\dom{\I_c}$. 
It is thus left to show that no other domain object is added to any of the closed concept names or closed role names in $\J$. First,
assume $w \in  \dom{\J} \setminus \dom{\I_c} $ is an object outside the core,
that is, $w$ is of the form $w = a\Ax_1\tau_1\ldots\Ax_\ell\tau_\ell$ with $\ell
\geq 1$, and $\tau_{\ell}$ is not a c-type.  Thus, $A \in \Sigma$
  implies $A \not\in\tau_{\ell}$, and thus by construction, $w \not \in A^\J$.
For a role name $r\in \Sigma$, that no domain object other than individuals belong to $r$ in $\J$ is ensured by the fact that inclusions of the form $\Ax_i = A\ISA \Some{r_i}{A'} \in \T$  with
$r_i\,{\in_\T}\,\Sigma$, that is $r_i\suprr r$ with $r^{(-)}\in \Sigma$, do not appear in any run in $\srun$ or $\run$. Note that by definition a type containing $A$ is a c-type, and runs in $\srun$ cannot continue after a c-type. Moreover, runs start from fringe individuals which, by definition of a core, cannot realize c-types.

It remains to prove that $\J$ models $\K$, that is  $\J \models_\Sigma\A$ and  $\J \models \T$. First, $\J \models_\Sigma\A$ is a direct consequence of the fact that  $\I_c \models_\Sigma\A$  and that $\J$ is an extension of $\I_c$.
To prove $\J \models \T$, we show that $\J$ satisfies all inclusions of the
forms $\mathbf{(N1)}$, $\mathbf{(N2)}$, $\mathbf{(N3)}$, and $\mathbf{(N4)}$
in $\T$. To improve readability, this more technical part of the proof is
given in the appendix.

For ``$\Leftarrow$'', assume an arbitrary model $\J$ of $\K$ that is an
extension of $\I_c$. We show that we can extract from it a non-losing strategy for Bob. First, let $T$ be the set of all the types realized in $\J$, and
observe that $T \subseteq \LC(\T,\Sigma,\I_c)$ and that for all $a \in
\indivnames(\K)$, $\type(a, \J) \in T$. 
\smallskip

We define a strategy $\str$ as follows:
\begin{enumerate}[(*)]
 \item For each $\tau \in T$ that is not a c-type and each $A \ISA \Some{r}{A'} \in \T$, with $A \in \tau$, 
  we set $\str(\tau,A \ISA \Some{r}{A'}) = \tau'$ for an arbitrarily chosen
  $\tau' \in T$ that satisfies (C1) and (C2). 


\end{enumerate}
 The preceding is possible because $\J$ is a model, hence it satisfies all
 existential and universal inclusions in $\T$, and because all types realized in $\J$
 are contained in $T$. More precisely, assume an arbitrary type $\tau \in T$ that is not a c-type, let $c$ be an arbitrary object that realizes $\tau$ in $\J$ and assume an inclusion $A \ISA \Some{r}{A'} \in \T$. 
We fix $\str(\tau,A \ISA \Some{r}{A'})$ to be the type $\tau'$ that is realized by $c'$, for an arbitrary object $c' \in \dom{\J}$ that satisfies
$c' \in A'^{\J}$ and 
$(c,c') \in r^{\J}$. 
The existence of such a type $\tau'$ is guaranteed from the fact that $c \in A^{\J}$ and that $\J$ satisfies $ A \ISA \Some{r}{A'}$. Clearly, $A' \in \tau'$, that is $\tau_i$ satisfies (C1). In addition, $\J$ being a model of $\K$, and in particular satisfying all universal and role inclusions of $\K$, guarantees that $\tau'$ also satisfies (C2). More precisely, assume an arbitrary inclusion $ A_1\ISA \All{s}{A_2} \in \T$ such that $r  \suprr s$ and $A_1\in \tau$. We have to show that $A_2\in \tau'$ is the case. First, $c \in A_1^{\J}$ and $(c,c') \in s^\J$ are direct consequences of the assumption that $\tau = \type(c,\J)$ and the fact that $\J \models r \sqsubseteq s$ holds. 
 The latter together with the fact that $\J \models A_1\ISA \All{s}{A_2}$ imply that $c' \in A_2^\J$. It thus follows $A_2 \in \tau'$ since $\tau' = \type(c_1,\J)$. Similarly, assume an arbitrary inclusion $ A_1\ISA \All{s}{A_2} \in \T$ such that $r^- \suprr s $ 
 and $A_1\in \tau'$. It follows that $(c',c) \in s^{\J}$ and $c' \in A_1^{\J}$. Consequently, $c\in A_2^\J$ and therefore $A_2 \in \tau$ since $\type(c,\J)=\tau$.

Finally, one can 
see that \str is a non-losing strategy for Bob on $\I_c$.  
For every run $w$ that follows \str,
 the fact that $T \subseteq \LC(\T,\Sigma,\I_c)$ guarantees that $\tail(w) \in
 \LC(\T,\Sigma,\I_c)$. 
Condition  (*) ensures that 
$\str(\tail(w),A \ISA \Some{r}{A'})$   is defined
for every $A \ISA \Some{r}{A'} \in \T$ with $A \in \tail(w)$. 
\end{proof}

By putting Lemma~\ref{thm:decomp} and Proposition \ref{p:non-losing} together,
we obtain the 
main result of this section, stated in the following theorem.
 
\begin{theorem}\label{th:non-losing}
Assume a c-safe OMQ $Q=(\T,\Sigma, \q(\vec{x}))$ and let $\K=(\T,\Sigma, \A)$ be a KB where $\A$ is an ABox over the concept and role names occurring in $\T$. 
 Then $\K\not\modelsind \q(\vec{a})$ for a tuple of individuals $\vec{a}$ occurring in $\K$ iff there exists a core $\I_c$ for $\K$ such that
 \begin{enumerate}[\it(1)]
     \item \label{en:1}$\I_c \not\models \q(\vec{a}) $, and 
     \item \label{en:2} there is a non-losing strategy for Bob on $\I_c$. 
\end{enumerate}
\end{theorem}

\begin{example} Consider again our running example and consider the assertion $r_1(a,a)$. We can see that $\K \not\models r_1(a,a)$. This is witnessed by the core 
$\I_c^1$, which (unlike $\I_c^2$) does not satisfy $r_1(a,a)$, 
and the non-losing strategy for Bob on $\I_c^1$ presented above. 
\end{example}

\section{The Marking Algorithm}\label{sec:marking-algorithm}

In this section, to decide whether Bob has a non-losing strategy on a given core
 we use the type elimination procedure $\mathbf{Mark}$ in Algorithm 1, 
which  \emph{marks} (or \emph{eliminates}) all types from which Sam
has a strategy to defeat Bob.
It takes as input the TBox $\T$, the set $\Sigma\subseteq\conceptnames\cup\rolenames$, and a core $\I_c$ for some KB $\K=(\T,\Sigma, \A)$.  
The algorithm starts by building the set $N$ of all possible types over $\T$, 
and 
then it  marks types that are not good choices for Bob. 
In steps
(M$_{\mathrm{N}1}$) and (M$_{{\Sigma}}$) the algorithm respectively marks in
$N$ all types that violate 
the conditions (LC$_{\mathrm{N}1}$) or (LC$_{\Sigma}$); Bob is not allowed to
choose those types.  
Then, in the loop, (M$_\exists$)  exhaustively marks 
 types $\tau$ that, if picked by Bob, would allow Sam to pick an inclusion $A\ISA \Some{r}{A'}$ 
for which Bob cannot reply with any $\tau'$. 

  \LinesNotNumbered
  \begin{algorithm}[t]
    \DontPrintSemicolon \SetKwInOut{Input}{input}
    \SetKwInOut{Output}{output} \smallskip \Input{a TBox $\T$, a set  $\Sigma\subseteq\conceptnames\cup\rolenames$, and a core $\I_c$ for a KB $(\T,\Sigma, \A)$} 
      \smallskip \Output{Set of (possibly marked) types}

 \BlankLine\text{} $N \leftarrow \types$ 
 
  \medskip (M$_{\mathrm{N}1}$) Mark each $\tau \in N$ 
    violating an inclusion of the form $\mathbf{(N1)}$ in $\T$. \; 
 
\medskip
    (M$_{{\Sigma}}$) Mark each  c-type $\tau \in N$  that is not
    realized in $\I_c$. 

    \smallskip \Repeat{no new type is marked}{

\medskip
      (M$_\exists$) Mark each $\tau \in N$ such that $A\ISA
      \Some{r}{A'} \in \T$, $A\in \tau$, 
      and \emph{for each}~$\tau'\in N$, at least one the following holds:
      
\medskip
      \qquad(C0)~~ $\tau'$ is marked,

\smallskip
      \qquad (C1$^\prime$)~~ $A' \notin \tau'$, or

\smallskip
      \qquad (C2$^\prime$)~~ there exists $A_1\ISA \All{s}{A_2} \in \T$
        with
        
\medskip
        \qquad\qquad (i) $r \suprr s$ and $A_1\in \tau$ and $A_2\notin\tau'$, or

\smallskip
        \qquad\qquad (ii) $r^- \suprr s $ and $A_1 \in \tau'$ and $A_2
          \notin \tau$
        \medskip  }
 \smallskip

        \Return{$N$} \BlankLine
    \caption{ $\mathbf{Mark}$}
    \label{alg:dl_perfectrefbgp}
  \end{algorithm}
  
The correspondence between types that are marked by the algorithm, and those
that may occur in runs that follow a non-losing strategy is provided in the following
proposition.

\begin{proposition}\label{prop:mark} Consider a KB $\K=(\T,\Sigma,\A)$, 
a core $\I_c$ for $\K$, and a non-losing strategy \str for Bob
on $\I_c$. For every $\tau \in \types$, if $\tau$ is marked by
$\mathbf{Mark}(\T, \Sigma,\I_c)$, then there is no run $w$ following 
\str with $\tail(w) = \tau$. 
\end{proposition}

\begin{proof}  Let $\str$ be a non-losing strategy on $\I_c$. 
Consider an arbitrary  $\tau$ marked by 
$\mathbf{Mark}(\T, \Sigma,\I_c)$. 
We denote by 
$N_0$ 
the set of types that are marked before the loop,  by
(M$_{\mathrm{N}1}$), or (M$_{\Sigma}$); 
and we denote by $N_i$ the types that are marked in the $i$-th iteration of
the loop.  We show by induction on  $k$ that, if $\tau \in
N_k$, then there is no run $w$ following  
$\str$ with $\tail(w) = \tau$. 

For the base case, $k=0$, $\tau \in N_0$ implies that $\tau$ is marked by one of 
(M$_{\mathrm{N}1}$), or (M$_{\Sigma}$).  
Then $\tau$ does not satisfy (LC$_{\mathrm{N}{1}}$) or (LC$_\Sigma$),
which implies that $\tau \notin \LC(\T,\Sigma,\I)$ and
hence there cannot exist a run $w$ with $\tail(w)=\tau$, since this would
contradict the definition of non-losing $\str$. 

For the inductive case, let $\tau \in N_k$.
Let $\alpha = A \ISA \exists r.A' \in T$ with $A \in \tau$ be the axiom chosen for
(M$_\exists$) that results in the marking of $\tau$. 
Towards a contradiction, assume there exists a run $w$ that follows $\str$ and
$\tail(w) = \tau$. 
By definition of non-losing strategies, $\str(\tau,\alpha) = \tau'$ must be defined, and
it must be  such that: 
\begin{compactitem}\item[($\star$)] $\tau' \in \LC(\T,\Sigma,\I_c)$, $A'\in \tau'$, and
  for all inclusions $A_1\ISA \All{s}{A_2} \in \T$: 
  \begin{asparaenum}[-]
  \item if $r\suprr s $ and $A_1\in \tau$ then $A_2\in\tau'$,
  \item if $r^-\suprr s $ and $A_1 \in \tau'$ then $A_2 \in  \tau$.
    \end{asparaenum} 
\end{compactitem}
Then, since $w$ is a run that follows $\str$, and $\str(\tau,\alpha) = \tau'$,
then $w' = w\alpha\tau'$ is also a run that follows $\str$. 
However, ($\star$) implies that neither  (C1$'$) nor (C2$'$) hold for $\tau'$. Hence, since $\tau$ gets marked by (M$_\exists$), then (C0) must hold, that
is,  $\tau' \in N_j$ for some $j < k$. The latter together with the induction hypothesis imply that
$\tau'$ cannot be the tail of a run that follows $\str$, and thus 
$w' = w\alpha\tau'$ cannot be a run that follows $\str$, which is a
contradiction.
\end{proof}

Now we can formally establish how the marking algorithm allows us to verify
  the existence of a non-losing strategy for Bob on a given core. 

\begin{theorem} \label{th:mark}
Consider a KB $\K=(\T,\Sigma,\A)$, and
a core $\I_c$ for $\K$. 
Then Bob has a non-losing strategy on $\I_c$ iff there is no fringe individual $c^{\alpha}$ in $\I_c$ whose type $\type(c^{\alpha}, \I_c)$ is marked by  $\mathbf{Mark}(\T, \Sigma,\I_c)$.
\end{theorem}

\begin{proof}
For the ``$\Rightarrow$'' direction, we show that
if  there is
some fringe individual $a$ in $\I_c$ that realizes a type that is marked by $\mathbf{Mark}(\T,
\Sigma,\I_c)$, then 
 there is no non-losing strategy for Bob on $\I_c$. 
Let $a$ be a fringe individual such that $\tau=\type(a,\I_c)$ is marked by
$\mathbf{Mark}(\T, \Sigma,\I_c)$.
Towards a contradiction, assume that a non-losing strategy
 \str on $\I_c$ exists. But by definition, $a$ is a run that follows
\str, and  since $\tau=\tail(a)$ is marked by
$\mathbf{Mark}(\T, \Sigma,\I_c)$, by Proposition \ref{prop:mark}, $\tau$ cannot occur in a run that follows $\str$,  a contradiction. 

\smallskip

For the ``$\Leftarrow$'' direction, let $T$ be the set of all types that are
not marked by the algorithm $\mathbf{Mark}(\T,\Sigma,\I_c)$;  note that $T \subseteq \LC(\T,\Sigma,\I_c)$. 
Assume that  $\type(a,\I_c) \in T$ for every fringe individual $a$ in $\I_c$.
 
Then we can build a strategy $ \str$ on $\I_c$ as follows.
\begin{enumerate}
\item[(*)] For each pair of $\tau \in T$ that is not a c-type and $\alpha=A \ISA \exists r.A' \in \T$,  with $A\in \tau$, 
we set $\str(\tau,\alpha)$ to an arbitrary  $\tau' \in T$ that satisfies (C1) and (C2).
\end{enumerate} 
  
\smallskip

\noindent 
Note that such a construction of $\str$ is possible, since a suitable $\tau'$
always exists. If this were not the case for some pair 
then $\tau$ would clearly be marked by the algorithm at some iteration of
(M$_\exists$). 

We claim that $\str$ is a non-losing strategy on $\I_c$. Towards a
contradiction, assume it is not, and let
$w$ with $\tail(w) = \tau$ be a run that follows $\str$. Then either $\tau
\notin \LC(\T,\Sigma,\I_c)$, or $\tau$ is not a c-type and  
there exists some inclusion $\alpha = A \ISA
\exists r.A' \in \T$, with $A \in \tau$, 
and $\str(\tau,\alpha )$ 
is not defined. The former contradicts the fact that $\tau \in T$ and thus $T \subseteq  \LC(\T,\Sigma,\I_c)$.
The latter contradicts the existence of $\tau'$ argued above.   
 Hence,  $\str$ is a
 non-losing strategy for Bob on $\I_c$.
\end{proof}

 \begin{example} 
Consider again our running example. 
Four concept names and a nominal occur in $\conceptnames^+(\T)$, namely
$A_1$, $A_2$, $A_3$, $A_4$, and $\{c\}$. Consequently, there are $2^5 = 32$
possible types over $\T$. 
The  algorithm marks the following types (recall that we omit $\top$ to avoid clutter): 
\begin{itemize}
\item
Every type containing $A_2$ but not $A_3$ or $A_4$ will be marked by 
(M$_{\mathrm{N}1}$). 
\item 
All types that contain $A_1$ or $A_4$ other than $\{A_1,A_4\}$ and
$\{A_1,A_3\}$,
and all types that contain  $\{c\}$ except for $\{\{c\}\}$ 
are  marked by (M$_\Sigma$). 
\end{itemize}
Of the  still unmarked types $\{\}$, $\{A_3\}$, and $\{A_2,A_3\}$, 
$\{A_1,A_4\}$, $\{A_1,A_3\}$, and $\{\{c\}\}$, none  gets marked in the loop, 
so they all remain unmarked.
\end{example}

\section{Rewriting into Datalog with Negation}\label{sec:rewriting}

\newcommand{\lw}{\textbf{**Label warning!**}}
\newcommand{\warning}[1]{\textbf{**#1**}}

In this section, we provide a rewriting of a given c-safe
ontology-mediated query $Q$ into 
\datalog 
extended with negation, under the stable model semantics. The program
in essence implements the marking algorithm from Section
\ref{sec:marking-algorithm}, and the output of the program (in terms
of the so-called \emph{skeptical entailment}) corresponds to the certain
answer to $Q$ over any input ABox. We later show that in the absence
of closed predicates we can provide a rewriting into a disjunctive
\datalog program without negation. We start by recalling the relevant
variants of \datalog.

\smallskip
\noindent
\subsection{\datalog with Disjunction and Negation}\quad We assume
countably infinite sets $\preds$ and $\varnames$ of \emph{predicate
  symbols} (each with an associated \emph{arity}) and
\emph{variables}, respectively. 
We further assume that $\conceptnames\cup \rolenames\subseteq \preds$
with each $A\in \conceptnames$ being unary, and each $r\in \rolenames$
being binary.  An \emph{atom} is an expression of the form
$R(t_1,\ldots,t_n)$, where
$\{t_1,\ldots,t_n\}\subseteq \indivnames\cup \varnames$, and $R$ is an
$n$-ary relation symbol. A \emph{negated atom} is an expression of the
form $\naf \alpha$, where $\alpha$ is an atom.  A (negated) atom is
\emph{ground} if it contains no variables, that is,
$\{t_1,\ldots,t_n\}\subseteq \indivnames$.  A \emph{rule} $\rho$ is an
expression of the form
\[ h_1\lor \ldots\lor h_n \gets b_1,\ldots,b_k\] where $n,k \geq 0$,
$H=\{h_1,\ldots,h_n\}$ is a set of atoms, called the \emph{head} of
$\rho$, and $B=\{b_1,\ldots,b_k\}$ is a set of possibly negated atoms,
called the \emph{body} of $\rho$. Each variable that occurs in $\rho$
must also occur in a (non-negated) atom in the body of $\rho$. If
$|H|>1$, then we call $\rho$ a \emph{disjunctive} rule. If negated atoms do
not occur in $\rho$, then $\rho$ is \emph{positive}. We say a rule
$\rho$ is a \emph{constraint} if $H=\emptyset$, that is if $\rho$ is
of the form $\gets b_1,\ldots,b_k$.  Rules of the form $h\gets$ (known
as \emph{facts}) are simply identified with the atom $h$, thus ABox
assertions are valid facts in our syntax. For a role name $p$, we may
use $p^-(t_1,t_2)$ to denote the atom $p(t_2,t_1)$. A \emph{program}
is any finite set $P$ of rules. If no disjunctive rules occur in $P$,
we call $P$ a \emph{non-disjunctive} program. Unless specified
otherwise, we consider non-disjunctive programs by default.  If all rules in $P$
are positive, then $P$ is also called \emph{positive}.  We use
$\ground(P)$ to denote the \emph{grounding} of $P$, i.e.\,the
variable-free program that is obtained from $P$ by applying on its
rules all the possible substitutions of variables by individuals of
$P$.

A \emph{(Herbrand) interpretation} $I$, also called  a \emph{database}, 
is any finite set of variable-free (or \emph{ground}) atoms.  We assume a
binary built-in \emph{inequality} predicate
$\neq$ with a natural meaning: in any interpretation $I$, $a\neq b\in I$ iff
\begin{inparaenum}[(i)]
\item $a\neq b$, and
\item both $a,b$  occur in some atoms $R_1(\vec{t}_1),R_2(\vec{t}_2)\in I$, where neither of $R_1,R_2$ is the inequality predicate.
\end{inparaenum}

An interpretation $I$ is a \emph{model} of a positive program $P$ if
$\{b_1,\ldots,b_k\}\subseteq I$ implies
$I\cap \{h_1,\ldots,h_n\}\neq \emptyset$ for all rules
$h_1\lor \ldots\lor h_n \gets b_1,\ldots,b_k$ in $\ground(P)$. We say
an interpretation $I$ is a \emph{minimal model} of a positive program
$P$ if $I$ is a model of $P$, and there is no $J\subsetneq I$ that is
a model of $P$.

The \emph{GL-reduct} of a program $P$ with respect to \,an
interpretation $I$ is the program $P^I$ that is obtained from
$\ground(P)$ in two steps~\cite{gelf-lifs-88}:
\begin{compactenum}[(i)]
\item deleting every rule that has $\naf \alpha$ in the body with
  $\alpha\in I$, and
\item deleting all negated atoms in the remaining rules.
\end{compactenum} An interpretation $I$ is a \emph{stable model} (also
known as an \emph{answer set}) of a program $P$ if $I$ is a minimal
model of $P^I$.

We call \emph{query} a pair $(P,q)$ of a program $P$ and a predicate
symbol $q$ occurring in $P$. A tuple $\vec{a}$ of constants is a
\emph{certain answer to} $(P,q)$ \emph{over a database} $I$ if
$q(\vec{a})\in J$ for all stable models $J$ of $P\cup I$; the set of
all such $\vec{a}$ is denoted $\cert((P,q),I)$. A ground atom
$q(\vec{a})$ is \emph{entailed} from a program $P$ and a database $I$,
written $(P,I) \models q(\vec{a})$, if $\vec{a}$ is a certain answer
to $(P,q)$ over $I$.


\subsection{Rewriting OMQs}

In this section, for a given c-safe ontology-mediated query
$Q=(\T, \Sigma, \q)$, where $\T$ is an $\ALCHOI$ TBox, we build a
query $(P_Q,q)$, where $P_Q$ is a  \datalog
program with negation,
such that for every ABox $\A$ over the concept and role names
occurring in $\T$ and for every tuple of constants $\vec{a}$, we have:
\[ (\T,\Sigma,\A) \models \q(\vec{a}) \mbox{~iff~} \vec{a} \mbox{~is a
    certain answer to~} (P_Q,q) \mbox{~over~} \A. \] We later show
that in case $\Sigma=\emptyset$, we can obtain a positive disjunctive
program, which might employ the inequality predicate $\neq$. If
additionally there are no nominals in $\T$, then $\neq$ is not
necessary.  Crucially, in all cases $(P_Q,q)$ can be built in
polynomial time, and is independent of any input ABox.

The main challenge to obtain the desired program $P_Q$ is to build a
program whose stable models correspond to the cores that can be
extended into a model of $(\T,\Sigma,\A)$, for any input ABox $\A$.
More precisely:

 \begin{proposition}\label{prop:building-program}
   Let $\T$ be a TBox, let $\Sigma$ be a set of closed predicates, and
   let $\A$ be an ABox over the concept and role names occurring in
   $\T$.  Then we can build in polynomial time a program $P^*_Q$ such
   that the following hold:
   \begin{enumerate}[-]
   \item If $\I_c$ is a core for $(\T,\Sigma,\A)$ that can be extended
     into a model of $(\T,\Sigma,\A)$, then there is a stable model $I$ of
     $\A \cup P^*_Q$ such that
     \begin{inparaenum}[(i)]
     \item $A(b)\in I$ iff $b\in A^{\I_c}$ for all concept names $A$, and
     \item $r(b,c)\in I$ iff $(b,c)\in r^{\I_c}$ for all role names
       $r$.
     \end{inparaenum}

   \item If $I$ is a stable model of $\A \cup P^*_Q$, then there
     exists a core $\I_c^{I}$ for $(\T,\Sigma,\A)$ that can be
     extended into a model of $(\T,\Sigma,\A)$, and such that
     \begin{inparaenum}[(i)]
     \item $A(b)\in I$ iff $b\in A^{\I_c}$ for all concept names $A$, and
     \item $r(b,c)\in I$ iff $(b,c)\in r^{\I_c}$ for all role names
       $r$.
     \end{inparaenum}
   \end{enumerate}
 \end{proposition}

Before proving Proposition~\ref{prop:building-program}, we state
our main result on rewritability. It is a direct
 consequence of Proposition~\ref{prop:building-program} and the
 previously presented characterization of OMQ answering.
 
 \begin{theorem} For a c-acyclic OMQ $Q=(\T, \Sigma, \q)$, where $\T$
   is an $\ALCHOI$ TBox, we can build in polynomial time a query
   $(P_Q,q)$, where $P_Q$ is a \datalog program with stable negation,
   such that $\cert(Q,\A) = \cert((P_Q,q), \A)$ for any given ABox
   $\A$ over the concept and role names occurring in $\T$.
 \end{theorem}
 \begin{proof}
   This is an immediate consequence of
   Proposition~\ref{prop:building-program}, together with Lemma
   \ref{l:acytogr}, and Theorems \ref{th:non-losing} and
   \ref{th:mark}. To obtain a polynomial time rewriting from a
   c-safe OMQ $Q=(\T, \Sigma, \q(\vec{x}))$, we simply need to
   compute $P^*_Q$, pick a fresh predicate $q$ of the same arity as
   $\q$, and add to $P^*_Q$ an additional rule that, for every database
   $D$, includes the atoms $q(\vec{c})$ in every stable model of $P_Q$
   over $D$ whenever $\vec{c}$ is a certain answer of $Q$ over
   $\A$. We simply let $P_Q=P^*_Q\cup \{q(\vec{x}) \gets \q\} $.
 \end{proof}

 We will
 dedicate the remainder of this section to prove Proposition~\ref{prop:building-program} by building 
the program $P^*_Q$, which is defined as the union of three components:
 \begin{enumerate}
 \item[$P_c$] is a set of rules that, given an input ABox $\A$,
   non-deterministically generates all possible cores for
   $(\T,\Sigma,\A)$.
 \item[$P_M$] is a set of rules that implement the type elimination
   algorithm presented in Section~\ref{sec:marking-algorithm}.
 \item[$P_T$] is a set of rules 
   that, relying on the marking done by $P_M$, filters out from the
   cores generated by $P_c$ those that cannot be extended into a
   model of $(\T,\Sigma,\A)$, since they contain some fringe individual whose type is marked. 
 \end{enumerate}
 We construct $P_c$, $P_M$, and $P_T$ next.

 \subsubsection{Generating the cores}



Let $\T$ be a TBox, let $\Sigma$ be a set of closed predicates, and
   let $\A$ be an ABox over the concept and role names occurring in
   $\T$.  We start by describing $P_c$. Its main property is that, given an
 input ABox $\A$, the stable models of $\A \cup P_c$ correspond to the
 cores for $(\T,\Sigma,\A)$: if $I$ is a stable model
 of $\A \cup P_c$, then we can extract a core for $(\T,\Sigma,\A)$ from it,
 and conversely, if $\I_c$ is a core for $(\T,\Sigma,\A)$, then we can
 define from it a stable model of $\A \cup P_c$.


\paragraph{\bf (I) Collecting the individuals}\quad We start with 
rules to collect in the unary predicate $\mathsf{ind}$ all the
individuals that occur in $\T$ or the input ABox $\A$; this allows us
to compute the set $\indivnames(\K)$. For all
nominals
$\{a\}$ that occur in $\T$, for all $A \in \conceptnames(\T)$, and for all $p\in\rolenames(\T)$, and  we add to $P_c$:
\begin{align*}
  \mathsf{ind}(a)&\gets &  \mathsf{ind}(x) &\gets p(x,y) \\
\mathsf{ind}(x) & \gets A(x) & \mathsf{ind}(y) &\gets p(x,y)
\end{align*}
 
\paragraph{\bf (II) Generating the cores}\quad Our next goal is to
generate candidate cores. Recall that in cores, the set of original
individuals of the knowledge bases can be expanded with a limited number
of further elements. We need to simulate these elements, because they
are not directly available to us. For every existential inclusion
$\alpha$ in $\T$, let $\mathsf{in}^{\alpha}$ and
$\mathsf{out}^{\alpha}$ be two fresh unary relations.  We can now
guess the domain of our candidate core. For all existential inclusions
$\alpha$ in $\T$ we add the following rules:
\begin{align*}
  \mathsf{in}^{\alpha}(x) & \gets \mathsf{ind}(x),\naf \mathsf{out}^{\alpha}(x)  \\
  \mathsf{out}^{\alpha}(x) & \gets \mathsf{ind}(x),\naf \mathsf{in}^{\alpha}(x)
\end{align*}
Intuitively, the fact $\mathsf{in}^{\alpha}(c)$ means that the special
constant $c^{\alpha}$ is included in the domain of the candidate
core. To characterize the possible participation by the ordinary and
the special constants in the relevant concept names, for each
$A \in \conceptnames(\T)\setminus \Sigma$, and all of existential
inclusions $\alpha$ in $\T$, we take fresh unary predicates
$\overline{A}$, $A^{\alpha}$, $\overline{A}^{\alpha}$, and add the
following rules:
\begin{align*}
  A(x)  & \gets \mathsf{ind}(x),\naf \overline{A}(x) \\
  \overline{A}(x)  & \gets \mathsf{ind}(x), \naf A(x)\\
  A^{\alpha}(x)  & \gets \mathsf{in}^{\alpha}(x), \naf \overline{A}^{\alpha}(x) \\
  \overline{A}^{\alpha}(x) & \gets \mathsf{in}^{\alpha}(x), \naf  A^{\alpha}(x) 
\end{align*}
Intuitively, an atom $A^{\alpha}(c)$ mean that $c^{\alpha}$ is the
extension of $A$. To handle the participation in role names, for all
$p\in\rolenames(\T)\setminus \Sigma$ and existential inclusions
$\alpha$ in $\T$, we take a fresh binary predicate $\overline{p}$,
fresh unary predicates, $p^{\alpha}_{\rightarrow}$,
$p^{\alpha}_{\leftarrow}$, $\overline{p}^{\alpha}_{\rightarrow}$,
$\overline{p}^{\alpha}_{\leftarrow}$, and add the following rules:
\begin{align*}
  p(x,y)   & \gets \mathsf{ind}(x), \mathsf{ind}(y),\naf \overline{p}(x,y) \\
  \overline{p}(x,y) & \gets \mathsf{ind}(x), \mathsf{ind}(y), \naf  p(x,y)  \\
  p^{\alpha}_{D}(x)  & \gets \mathsf{in}^{\alpha}(x), \naf \overline{p}^{\alpha}_{D}(x) & D\in\{\rightarrow,\leftarrow\} \\
  \overline{p}^{\alpha}_{D}(x) & \gets \mathsf{in}^{\alpha}(x), \naf  p^{\alpha}_{D}(x)  & D\in\{\rightarrow,\leftarrow\}
\end{align*}
Intuitively, an atom $p^{\alpha}_{\rightarrow}(c)$ (resp.,
$p^{\alpha}_{\leftarrow}(c)$) states that there is an $r$-link from
$c$ to $c^{\alpha}$ (resp., from $c^{\alpha}$ to $c$).  Note that the
above rules are only for concept and role names not occurring in
$\Sigma$. In this way, the extension of the closed predicates in
$\Sigma$ is preserved in the stable models of $\A \cup P_c$, which is
necessary to satisfy the conditions of the core as in
Definition~\ref{def:coreint}.


   \paragraph{\bf (III) Validating cores}\quad We need to make sure
   that the structures generated by the rules in (I) and (II) satisfy
   the conditions (c3.1), (c3.2), (c3.3), (c3.4) and (c5) from
   Definition~\ref{def:coreint}. The remaining
   conditions are trivially satisfied.  
Next we use $\invat{r}(x,y)$
   to denote $p(x,y)$ if $r = p \in \rolenames$, and $p(y,x)$ if
   $r= p^-$ for some $p \in \rolenames$.
 For an inverse role $r=p^{-}$ with
   $p \in \rolenames$, $r_{\rightarrow}^{\alpha}$ denotes
   $p_{\leftarrow}^{\alpha}$.
   





   To deal with (c3.1), we take an auxiliary fresh binary relation $\mathsf{eq}$, and add to $P_c$ the rule $\mathsf{eq}(x,x)\gets ~\mathsf{ind}(x)$.  
Using  $\{o\}(x)$ to denote $\mathsf{eq}(x,a)$, and
   $\overline{\{o\}}(x)$ to denote $x\neq o$, we can add, for all inclusions
   $B_1\AND \cdots \AND B_n \ISA B_{n+1}\OR\cdots \OR B_{k}$ in $\T$, and all existential inclusions $\alpha$ in $\T$,
    the following constraints:
   \begin{align*}
     &  \gets~\mathsf{ind}(x),B_1(x),\ldots, B_n(x),\overline{B}_{n+1}(x),\ldots, \overline{B}_{k}(x) \\
     &  \gets~\mathsf{in}^{\alpha}(x),B_1^{\alpha}(x),\ldots, B_n^{\alpha}(x),\overline{B}_{n+1}^{\alpha}(x),\ldots, \overline{B}_{k}^{\alpha}(x)     
     \end{align*}

   To ensure the satisfaction of (c3.2),  for each $A~\ISA~\All{r}{B}\in\T$, and
   all existential inclusions $\alpha$ in $\T$ we add the following:
   \begin{align*}
     &  \gets~ A(x),\invat{r}(x,y), not~B(y)\\
     &  \gets~ A(x), r_{\rightarrow}^{\alpha}(x) ,not~B^{\alpha}(x) \\
     &  \gets~ A^{\alpha}(x), r_{\leftarrow}^{\alpha}(x) ,not~B(x) 
   \end{align*}

 To ensure (c3.3), for each $r \ISA s \in \T$, and all
   existential inclusions $\alpha$ in $\T$, we add the following:
   \begin{align*}
     &  \gets~\invat{r}(x,y), not~\invat{s}(x,y)  \\
     &  \gets~  r_{D}^{\alpha}(x),not~s_{D}^{\alpha}(x)     & D \in \{\rightarrow,\leftarrow\}      
   \end{align*}
%

   Let's deal with (c5). Assume an inclusion $A\ISA \exists r.B$ in
   $\T$. We take a fresh unary predicate $S$, and for all existential
   inclusions $\alpha$ in $\T$ we add the following:
   \begin{align*}
     S(x) &\gets~\invat{r}(x,y), B(y)\\
     S(x) & \gets~r_{\rightarrow}^{\alpha_1}(x),B^{\alpha_1}(x) \\
          &\gets~A(x),not~S(x)  
   \end{align*}

   Finally, to ensure (c3.4), for all existential inclusions $\alpha$ in
   $\T$, and all $A\ISA \Some{r}{A'}\in \T$ with \rin, we add the
   constraint $\gets A^{\alpha}(x)$.

This finishes the construction of $P_c$, whose stable models are in one-to-one correspondence with the cores for $(\T,\Sigma,\A)$. 
We  explain next how $P_M$ and $P_T$ are built. In the explanations, we may blur the distinction between the stable models of $P_c$
and the actual cores.

   \subsubsection{Implementing the algorithm $\mathbf{Mark}$}



   We now move on to $P_M$, which implements the algorithm
   $\mathbf{Mark}$ from Section \ref{sec:marking-algorithm}.  To
   obtain a polynomially sized program, we need to use non-ground
   rules whose number of variables depends on the number of different
   concept names and nominals in $\conceptnames^{+}(\T)$.  In a
   nutshell, in the grounding of $P_M$, types are represented as atoms
   of arity $|\conceptnames^{+}(\T)|$. The rules of $P_M$ generate
   these atoms and execute the marking algorithm on them: test for
   local consistency and for realization if required, iterate over all
   the types to search for suitable successors, etc.  In general, the
   rules of $P_M$ use predicates with large arities, depending
   linearly on $|\conceptnames^{+}(\T)|$.

   To represent types as ground atoms, we assume an arbitrary but
   fixed enumeration $B_1,\ldots,B_k$ of $\conceptnames^{+}(\T)$.
   $P_M$ relies heavily on the availability of at least two
     individuals.  Here, we use a special pair of constants $0,1$ that
     we add to the program using ground facts; note that our target
     language for rewriting OMQs is non-monotonic disjunctive \datalog
     with constants. 
     The constants can be easily omitted if we assume the availability
     in the database of at least two
     distinct 
     constants, which can be used in the place of $0,1$.  We will come
     back to the role of these constants in Section
     \ref{sec:dis}. 

   We now describe $P_M$. Below, the $k$-ary relation
   $\mathsf{Type} = \{0,1\}^k$ stores the set of all types over
   $\T$. Naturally, a $k$-tuple $(b_1,\ldots,b_k) \in \mathsf{Type}$
   encodes the type $\{B_i\mid b_i = 1,1 \leq i \leq
   k\}\cup\{\top\}$. The goal is to compute a $k$-ary relation
   $\mathsf{Marked}\subseteq \mathsf{Type}$ that contains precisely
   those types marked by the $\mathbf{Mark}$ algorithm.  $P_M$
   contains rules that compute $\mathsf{Type}$, $\mathsf{Marked}$, and
   some auxiliary relations.

   \paragraph{\bf (IV) A linear order over types}\quad
   \newt{We want the relation $\mathsf{Type}$ to contain all tuples in
     $\{0,1\}^k$. Of course, we could hard-code the list of these
     tuples, but there are exponentially many, thus the resulting
     program would not be polynomial.  What we do instead is to encode
     a linear order over $\{0,1\}^k$, and starting from the first
     element, populate $\mathsf{Type}$ by adding successors.  } To
   this end, for every $1\leq i \leq k$, we inductively define $i$-ary
   relations $\mathsf{first}^i$ and $\mathsf{last}^i$, and a $2i$-ary
   relation $\mathsf{next}^i$, which will provide the first, the last
   and the successor elements from a linear order on $\{0,1\}^i$. In
   particular, given $\vec{u},\vec{v}\in \{0,1\}^{i}$, the fact
   $\mathsf{next}^i(\vec{u},\vec{v})$ will be true if $\vec{v}$
   follows $\vec{u}$ in the ordering of $\{0,1\}^i$. The rules to
   populate $\mathsf{next}^{i}$ are quite standard (see, e.g., Theorem
   4.5 in \cite{DBLP:journals/csur/DantsinEGV01}). For the case $i=1$,
   we simply add the following facts:
   \[\mathsf{first}^1(0)\gets \qquad\quad \mathsf{last}^1(1)\gets
     \quad\qquad \mathsf{next}^1(0,1)\gets \] Then, for all
   $1<i \leq k-1$ we add the following rules:
   \begin{align*}
     \mathsf{next}^{i+1}(0,\vec{x},0,\vec{y}) &~\gets~ \mathsf{next}^i(\vec{x},\vec{y})\\
     \mathsf{next}^{i+1}(1,\vec{x},1,\vec{y}) &~\gets~ \mathsf{next}^i(\vec{x},\vec{y})\\
     \mathsf{next}^{i+1}(0,\vec{x},1,\vec{y}) &~\gets~ \mathsf{last}^i(\vec{x}),\mathsf{first}^i(\vec{y})\\
     \mathsf{first}^{i+1}(0,\vec{x}) &~\gets~ \mathsf{first}^i(\vec{x})\\ 
     \mathsf{last}^{i+1}(1,\vec{x}) &~\gets~ \mathsf{last}^i(\vec{x})
   \end{align*}
   We can now collect in the $k$-ary relation $\mathsf{Type}$ all
   types over $\T$:
   \begin{align*}
     \mathsf{Type}(\vec{x}) &~\gets~ \mathsf{first}^k(\vec{x}) \\
     \mathsf{Type}(\vec{y}) &~\gets~ \mathsf{next}^k(\vec{x},\vec{y})
   \end{align*}

   \newcommand{\bvec}{\mathit{vec}} $\mathsf{Type}$ is the set $N$ of
   the $\mathbf{Mark}$ algorithm.  {In the rest of this section, we
     will write $\bvec(\tau)$ to denote the bit vector encoding the
     type $\tau$.  In a slight abuse of terminology, we may say that
     \emph{we mark $\tau$} to mean that we enforce
     $\mathsf{Marked}(\bvec(\tau))$ to hold, and if no confusion
     arises, we may interchange $\tau$ and $\bvec(\tau)$.}

   Below we describe rules that simulate the algorithm on this $N$.
   Recall that the algorithm runs on a given core $\I_c$.  \newt{To
     implement the marking algorithm, 
     we first mark all the types $\tau$ that do not satisfy the local
     consistency conditions $\LC(\T,\Sigma,\I_c)$.
     This is done with the rules described in items (V) to (VIII)
     below.}

\paragraph{\bf (V) Implementing Step (M$_{\mathrm{N}1}$)}\quad
We use two auxiliary additional unary predicates $\mathsf{T}$ and
$\mathsf{F}$ (intuitively, \emph{True} and \emph{False}), and add two
facts to $P_M$:
\[ \mathsf{F}(0)\gets~ \qquad \mathsf{T}(1)\gets~\] For a $k$-tuple of
variables $\vec{x}$, we let $B\in\vec{x}$ denote the atom
$\mathsf{T}(x_j)$, where $j$ is the index of $B$ in the enumeration of
$\conceptnames^{+}(\T)$. Similarly, we let $B\not\in\vec{x}$ denote
the atom $\mathsf{F}(x_j)$, where $j$ is the index of $B$ in the
enumeration.

The step (M$_{\mathrm{N}1}$) of the algorithm, which marks types
violating inclusions of type $\mathbf{(N1)}$, is implemented using the
following rule in $P_M$, for every inclusion
$B_1\AND \cdots \AND B_n \ISA B_{n+1}\OR\cdots \OR B_k \in \T$:
\[{\mathsf{Marked}(\vec{x})} \gets {
    \mathsf{Type}(\vec{x}),B_1\in\vec{x},\ldots,B_n \in \vec{x}, }{
    B_{n+1}\not\in\vec{x},\ldots,B_k \not\in \vec{x}}\]

\paragraph{\bf (VI) Collecting the realized types}\quad 
Our next goal is to compute a $k$-ary relation $\mathsf{RealizedType}$
such that, for each core
$\I_c$, 
$\mathsf{RealizedType}$ will be populated with precisely the
$\{0,1\}^k$-vectors that represent a type realized in $\I_c$ by some individual in $\indivnames(\K)$. This
relation will be useful in the rest of the marking algorithm.
 
To compute $\mathsf{RealizedType}$, we use auxiliary $(i+1)$-ary
relations $\mathsf{hasType}^{i}$ for all $0 \leq i \leq k$.  We first
put all the individuals into $\mathsf{hasType}^0$; intuitively, every
individual has an `empty' type up to the $0$-th position.  Then we
iteratively take an individual whose type has been stored up to the
$(i-1)$-th position, and expand it to the $i$-th, using $0$ or $1$
according to whether or not the type that it realizes in $\I_c$
contains
$B_i$. 
We first add to $P_M$ the rule:
\[\mathsf{hasType}^0(x) \gets \mathsf{ind}(x)\]

\noindent and then, the following rules for all
$1 \,{\leq}\, i \,{\leq}\, k$:
\begin{align*}
  \mathsf{hasType}^{i}(x,\vec{y},1) & ~\gets~ \mathsf{hasType}^{i-1}(x,\vec{y}),B_i(x)     & B_i\in \conceptnames \\
  \mathsf{hasType}^{i}(a,\vec{y},1) & ~\gets~ \mathsf{hasType}^{i-1}(a,\vec{y})      & B_i =\{a\} \\
  \mathsf{hasType}^{i}(x,\vec{y},0) & ~\gets~ \mathsf{hasType}^{i-1}(x,\vec{y}),x\neq a     & B_i =\{a\} \\
  \mathsf{hasType}^{i}(x,\vec{y},0) & ~\gets~ \mathsf{hasType}^{i-1}(x,\vec{y}),\overline{B}_i(x) &  B_i\in \conceptnames\setminus \Sigma \\
  \mathsf{hasType}^{i}(x,\vec{y},0) & ~\gets~ \mathsf{hasType}^{i-1}(x,\vec{y}),not~B_i(x) &  B_i\in \conceptnames\cap  \Sigma
\end{align*}
\newt{Intuitively, $\mathsf{hasType}^{k}(c,\bvec(\tau))$ says that the
  individual $c$ realizes the type
  $\tau$. 
} We can now project away the individuals and store in the relation
$\mathsf{RealizedType}$ the set of all types realized in $\I_c$:
\[\mathsf{RealizedType}(\vec{y}) \gets \mathsf{hasType}^k(x,\vec{y})\]

\paragraph{\bf (VII) Implementing Step (M$_{\Sigma}$)}\quad In this
step, we mark all the non-realized types that are c-types.   
The following rules are added to $P_M$. First, we use a fresh $k$-ary
  predicate $ \mathsf{ClosedType}$ to collect all the c-types. In
  particular, for each $B$ such that (i) $B$ is a nominal from $\T$, (ii)  $B\in  \conceptnames(\T) \cap \Sigma$, or (iii)
 $\T$ contains an inclusion $ B\ISA \Some{r}{A} $ for some $A$ and  \rin, we add: 
\[ \mathsf{ClosedType}(\vec{x})\gets
  \mathsf{Type}(\vec{x}),B\in\vec{x}\] \newt{We can now mark the
  c-types that are \emph{not realized}, using a simple rule:}
\[ \mathsf{Marked}(\vec{x})\gets \mathsf{ClosedType}(\vec{x}),
  not~\mathsf{RealizedType}(\vec{x})\]

The rules in (V) to (VII) mark, for a core $\I_c$, the types
  over $\T$ that do not satisfy the local consistency conditions
  $\LC(\T,\Sigma,\I_c)$. Now we move to the next step: marking the
  types for which there is no suitable successor type.

\paragraph{\bf(VIII) Implementing Step (M$_\exists$)}\quad 
\newt{Consider an inclusion $\alpha = A\ISA \Some{r}{A'} \in \T$ with
  \rnotin.}
Recall that we need to mark a type $\tau$ if $A\in \tau$, and
\emph{for each} type $\tau' \in N$ at least one of (C0),
(C1$^\prime$), or (C2$^\prime$) holds.  To this aim we use, for each
such $\alpha$, an auxiliary $2k$-ary relation
$\mathsf{MarkedOne}_\alpha$ and collect, \newt{for each type $\tau$,
  all types $\tau'$ that cannot be used to satisfy
  $\alpha$}. \newt{That is, a pair $(\bvec(\tau),\bvec(\tau'))$
  will be in $\mathsf{MarkedOne}_\alpha$ if $\tau'$ is marked for
  $\tau$ by conditions (C0), (C1$^\prime$), (C2$'$).  Then we will
  iterate over our ordered list of types and test whether all $\tau'$
  are marked for $\tau$; in that case, we will mark $\tau$ if
  $A \in
  \tau$.} 

The following rules are added to $P_M$ for each $\alpha$ as above.  We
start with the rules for $\mathsf{MarkedOne}_\alpha$.


\begin{enumerate}[-]
\item 
  \newt{For (C0), we collect all pairs $(\bvec(\tau),\bvec(\tau'))$
    where $\tau'$ is marked:}
  \[\mathsf{MarkedOne}_\alpha(\vec{x},\vec{y})\gets
    \mathsf{Type}(\vec{x}), \mathsf{Marked}(\vec{y})\]

\item For (C1$'$), we collect pairs $(\bvec(\tau),\bvec(\tau'))$ with
  $A'\notin \tau'$:
  \[\mathsf{MarkedOne}_\alpha(\vec{x},\vec{y})\gets
    \mathsf{Type}(\vec{x}),\mathsf{Type}(\vec{y}),A' \not\in \vec{y}\]

\item For (C2$^\prime$), we proceed as follows.
  \begin{itemize}[-]
  \item For all $A_1\ISA~\All{s}{A_2}\in \T$ with $r\suprr s$, we
    collect all $(\bvec(\tau),\bvec(\tau'))$ such that $A_1 \in \tau$
    and $A_2 \notin \tau'$:
  \end{itemize}
  \[\mathsf{MarkedOne}_\alpha(\vec{x},\vec{y}) \gets
    \mathsf{Type}(\vec{x}),\mathsf{Type}(\vec{y}),A_1 \,{\in}\,
    \vec{x}, A_2 \,{\not \in}\, \vec{y}\]
  \begin{itemize}[-]
  \item For all $A_1\ISA~\All{s}{A_2}\in \T$ with $r^-\suprr s$, we
    collect all $(\bvec(\tau),\bvec(\tau'))$ such that $A_1 \in \tau'$
    and $A_2 \notin \tau$: 
  \end{itemize}
  \smallskip
  \[\mathsf{MarkedOne}_\alpha(\vec{x},\vec{y}) \gets
    \mathsf{Type}(\vec{x}),\mathsf{Type}(\vec{y}),A_1 \,{\in}\,
    \vec{y}, A_2 \,{\not\in}\, \vec{x}\]

\end{enumerate}

\noindent \newt{Now we want to infer $\mathsf{Marked}(\bvec(\tau))$ if
  $A$ is set to true in a type $\tau$, and
  $\mathsf{MarkedOne}_\alpha(\bvec(\tau),\bvec(\tau'))$ is true for
  all types $\tau'$.}  To achieve this, we rely on another auxiliary
$2k$-ary relation $\mathsf{MarkedUntil}_\alpha$ for each inclusion
$\alpha$:
\begin{align*}
  \mathsf{MarkedUntil}_\alpha(\vec{x},\vec{z})\gets & 
                                                      \mathsf{MarkedOne}_\alpha(\vec{x},\vec{z}),
                                                      \mathsf{first}^k(\vec{z}) \\
  \mathsf{MarkedUntil}_\alpha(\vec{x},\vec{u}) \gets &
                                                       \mathsf{MarkedUntil}_\alpha(\vec{x},\vec{z}),
                                                       \mathsf{next}^k(\vec{z},\vec{u}),\\
                                                    & \mathsf{MarkedOne}_\alpha(\vec{x},\vec{u})
\end{align*}

\noindent Intuitively, with the above rules we traverse all types
$\tau'$ checking $\mathsf{MarkedOne}_\alpha$ for the pair
$(\bvec(\tau),\bvec(\tau'))$.
We mark $\tau$ if we 
reach the last $\tau'$, and $A \in\ \tau$: 
\[\mathsf{Marked}(\vec{x})\gets
  \mathsf{MarkedUntil}_\alpha(\vec{x},\vec{z}), A \in \vec{x},
  \mathsf{last}^k(\vec{z})\]

\newt{This completes the rules of $P_M$. Now we move to the
  construction of $P_T$ which essentially uses the marked type from
  $P_M$ to forbid the cores, generated by $P_c$, that cannot be
  extended into a model. }

\subsubsection{Filtering out cores that cannot be extended}

Finally, the program $P_T$ filters out from the cores generated by
$P_c$ those that cannot be extended into a model of
$(\T,\Sigma,\A)$. More precisely, $P_T$ forbids a core $\I_c$ when a
type realized in $\I_c$ is marked by the rules in $P_M$.

\paragraph{\bf (IX) Forbidding marked types in the core}\quad
For Theorem \ref{th:mark}, we need to ensure that each type that is
realized in $\I_c$ by a fringe individual is \emph{not} marked by the
algorithm $\mathsf{Marked}$. For this we first compute a $k$-ary
relation $\mathsf{FringeType}$ such that, for each core $\I_c$,
$\mathsf{FringeType}$ will be populated with precisely the
$\{0,1\}^k$-vectors that represent a type realized by a fringe
individual in $\I_c$.

To compute $\mathsf{FringeType}$, we use auxiliary $(i+1)$-ary
relations $\mathsf{hasType}^{i}_{\alpha}$ for all $0 \leq i \leq k$
and all existential inclusions $\alpha$ in $\T$.  For each existential
inclusion $\alpha$ in $\T$, and for all $1 \,{\leq}\, i \,{\leq}\, k$,
we add to $P_M$ the following rules:
\begin{align*}
  \mathsf{hasType}_{\alpha}^0(x) & ~\gets~ \mathsf{ind}(x) \\
  \mathsf{hasType}_{\alpha}^{i}(x,\vec{y},1) & ~\gets~ \mathsf{hasType}_{\alpha}^{i-1}(x,\vec{y}),B_i^{\alpha}(x)     & B_i\in \conceptnames \\
  \mathsf{hasType}_{\alpha}^{i}(x,\vec{y},0) & ~\gets~ \mathsf{hasType}_{\alpha}^{i-1}(x,\vec{y})     & B_i =\{a\} \\
  \mathsf{hasType}_{\alpha}^{i}(x,\vec{y},0) & ~\gets~ \mathsf{hasType}_{\alpha}^{i-1}(x,\vec{y}),\overline{B}^{\alpha}_i(x) &  B_i\in \conceptnames\setminus \Sigma 
\end{align*}
Intuitively, $\mathsf{hasType}_{\alpha}^{k}(c,\bvec(\tau))$ says that
the individual $c^{\alpha}$ realizes the type $\tau$. We can now
project away the individuals and store in the relation
$\mathsf{FringeType}$ the set of all types realized by fringe elements
in $\I_c$. For all existential inclusions $\alpha$ in $\T$ we add:
\[\mathsf{FringeType}(\vec{y}) \gets
  \mathsf{hasType}_{\alpha}^k(x,\vec{y})\]

Finally, prohibiting marked types at fringe individuals is done by
adding to $P_T$ the rule:
\[\gets \mathsf{Marked}(\vec{x}), \mathsf{FringeType}(\vec{x})\]

This concludes the description of $P_T$, and hence of the program
$P^*_Q= P_c \cup P_M \cup P_T $ described in
Proposition~\ref{prop:building-program}.

%

It remains to argue that $P_Q$ is of size polynomial in the
  size of $\T$. Indeed, $P_c$ is linearly bounded by the size of
  $\conceptnames^+(\T)$, the number of role names in $\rolenames(\T)$
  and the number of inclusions that appear in $\T$; $P_M$ is bounded
  polynomially on the size of $\conceptnames^+(\T)$ (in particular the
  rules in (VII)) and in the number of inclusions in $\T$ (the rules
  in (VIII)); similarly as for $P_M$, $P_T$ is also bounded polynomially
  in the number of inclusions in $\T$.  


 \subsection{Complexity of evaluating the program}\label{sec:comp}

 In this section, for a given c-safe OMQ $Q=(\T,\Sigma,\q)$, we
 analyze the data and combined complexity of evaluating the program
 $P_Q$ for a given ABox $\A$. The decision problem associated to
 answering a query $(P_Q,q)$ is analogous to that of OMQs defined in
 Section \ref{sec:omq}, that is given a query $(P,q)$, where $P$ is a
 \datalog program with negation under the stable model semantics, a
 (possibly empty) tuple of individuals $\vec{a}$, and an ABox $\A$,
 decide whether $\vec{a} \in \cert((P,q), \A)$.

 \smallskip 
With our translation, we can obtain the  following upper bounds. The bounds themselves (which are tight) are minor variations of results in the literature
(e.g., \cite{DBLP:conf/ijcai/LutzSW13}), 
and the aim of the proposition is to show the adequacy of our technique.    

 \begin{proposition} Let $Q=(\T, \Sigma, \q)$ be a c-safe OMQ, where
   $\T$ is an $ \ALCHOI$ TBox.  The problem to decide if $\vec{a}$ is
   a certain answer of $(P_Q,q)$ for a given ABox $\A$ over the
   concept and role names that occur in $\T$ is in \exptime
   w.r.t.\,combined complexity, and in co\np w.r.t.\,data complexity.
 \end{proposition}

 \begin{proof}
   The co\np bound for data complexity follows easily
   since it is known that query answering for \datalog programs with
   negation under the stable model semantics is co\np-complete in data
   complexity (see
   e.g.,\,\cite{DBLP:journals/csur/DantsinEGV01}).  The result on the combined complexity does not
   follow directly from the complexity of query answering in \datalog
   with negation, which is
   co\nexptime-complete~\cite{DBLP:journals/tods/EiterGM97}. Thus we
   need to argue more carefully about the shape of the program that
   results from our rewriting. In the program $P_Q$ some of
   the predicates have small arities, and the negation is used in a
   restricted way. As we shall see next, for these reasons our
   programs fall into a class of programs that can be evaluated in
   (deterministic) exponential time.


   In the following, we say that a program $P$ defines a relation $R$,
   if $R$ appears in the head of a rule in $P$.  The program $P_Q$ can
   be partitioned into programs $P_1,P_2,P_3$ as follows:
   \begin{enumerate}[-]
   \item $P_1$ consists of all rules in (I), (II), and (III). $P_1$ is a program with at most two variables in
     each rule.
   \item $P_2$ consists of  the
     rules in (VI), which define the relations $\mathsf{hasType}^{i}$
     and the relation $\mathsf{RealizedType}$.
   \item $P_3$ consists of the remaining rules.
   \end{enumerate}
   Note that $P_2$ and $P_3$ do not define any relations used in
   $P_1$.
   The program $P_2$ only depends on $P_1$, that is, none of the
   relation symbols in $P_2$ 
   is defined in $P_3$. The negative atoms of $P_2$ only involve
   relations that are
   only defined in
   $P_1$. 
   Similarly, the negative atoms of $P_3$ only involve relations that
   are only defined by $P_1\cup P_2$. Assume a set $F$ of facts over
   the signature of $P_1$. Due to the above properties, the successful
   runs of the following non-deterministic procedure generate the set
   of all stable models of $P\cup F$:
   \begin{enumerate}[(S1)]
   \item Compute a stable model $I_1$ of $P_1\cup F$.
   \item Compute the least model $I_2$ of $I_1\cup P_2^{I_1}$. If
     $I_2$ does not exist due to a constraint violation, then return
     \emph{failure}.
   \item Compute the least model $I_3$ of $I_2\cup P_2^{I_2}$. Again,
     if $I_3$ does not exist, then return \emph{failure}. Otherwise,
     output $I_3$.
   \end{enumerate}
   Since $P_1$ has at most two variables in every rule, each stable
   model $I_1$ of $P_1\cup F$ is of polynomial size in the size of
   $P_1\cup F$, and the set of all such models can be traversed in
   polynomial space. For a given $I_1$, performing steps (S2) and (S3)
   is feasible in (deterministic) exponential time, because
   $P_2^{I_1}$ and the subsequent $P_2^{I_2}$ are ground
   disjunction-free positive programs of exponential size. It follows
   that computing the certain answers to $(P_Q,q)$ for any given ABox
   $\A$ over the concept and role names of $\T$ requires only
   deterministic exponential time.
 \end{proof}

\subsection{Obtaining Positive Programs}

We now discuss the case of OMQs $Q=(\T, \Sigma, \q)$ without closed
predicates, i.e.\,when $\Sigma=\emptyset$. We argue that in this
restricted case, we can  obtain a rewriting into a positive disjunctive
program. 
In addition, if nominals are not present in $\T$, we do not
even need the $\neq$ predicate. Towards this result, we observe that in case closed predicates are absent, we can  simplify the notion of cores, and thus also the rules required for  generating cores (in particular, the rules in (II) and (III)). We now present a simplified  definition of cores.

\begin{definition}
A \emph{(simple) core} for a KB  $\K=(\T,\A)$ is an interpretation $\I_c =
(\dom{\I_c},\cdot^{{\I_c}})$, where 
  \begin{enumerate}[(c1)]
  \item $\dom{\I_c} = \indivnames(\K)$,  $a^{\I_c} = a$ for all $a\in \dom{\I_c}$, and $\cdot^{{\I_c}}$ is an interpretation function with
  \item $\I_c \models \A$, and
  \item $\I_c\models \alpha$ for each $\alpha \in \T$ of type (N1), (N3) and (N4), i.e.\,for all but existential inclusions in $\T$.
  \end{enumerate}
  The constants from $\indivnames(\K)$ are the \emph{fringe individuals} of $\I_c$.
 \end{definition}
 Using the above definition instead of Definition~\ref{def:coreint},
 and assuming the absence of closed predicates, all proofs of Section
 5 hold as is. It thus remains to appropriately modify the program
 constructed in Section 6.2 to accommodate the simplified definition of
 cores.  As before, given a TBox $\T$, for each
 $A \in \conceptnames(\T)$, we use a fresh unary predicate
 $\overline{A}$, and for each $p\in\rolenames(\T)$, a binary
 $\overline{p}$. We use $\invat{r}(x,y)$ to denote $r(x,y)$, if
 $r \in \rolenames$, and $r(y,x)$ if $r^- \in \rolenames$. Then the
 rules in (II) and (III) are replaced by the following disjunctive
 rules:
     \begin{align*}
   A(x)\lor \overline{A}(x) ~\gets &  ~  \mathsf{ind}(x) & & \mbox{for all~}A \in \conceptnames(\T), \\[1ex]
   p(x,y)\lor \overline{p}(x,y) ~\gets& ~ \mathsf{ind}(x),\mathsf{ind}(y)  & &
                                                                              \mbox{for all~}  p\in\rolenames(\T)\\[1ex]
    ~\gets & ~\mathsf{ind}(x),B_1(x),\ldots, B_n(x)  & &   \mbox{for all~}    B_1\AND \cdots \AND B_n \ISA   \\
     & ~\overline{B}_{n+1}(x),\ldots, \overline{B}_{k}(x)  & &
                                                                               \ISA B_{n+1}\OR\cdots \OR B_{k} \in\T  \\[1ex] 
   A'(y) ~\gets & ~ A(x),\invat{r}(x,y) & &\mbox{for all } A~\ISA~\All{r}{A'}\in\T \\[1ex]
   \invat{s}(x,y)  ~\gets &~ \invat{r}(x,y) & & \mbox{for all } r \ISA s \in \T
 \end{align*}

 The rules in (VII) are replaced using the following rules.  They mark every type that
contains a nominal, but the type is not realized in the current core.
Note that, for each core $\I_c$ and each nominal $\{a\}$, the
  only type containing $\{a\}$ that is realized in $\I_c$ is the
  actual type $\tau=\type(a,\I_c)$ of $a$, which is in fact stored in
  the atom $\mathsf{hasType}^k(a,\bvec(\tau))$.  For this reason, to
  mark the nominal types that are not realized, it suffices to simply
  mark every type $\tau$ with $\{a\} \in \tau$ and
  $\tau \neq \type(a,\I_c)$. We can achieve this by adding to $P_M$
the following rules for all $B \in \conceptnames^{+}(\T)$ and all
nominals $\{a\} \in \conceptnames^{+}(\T)$:
\begin{align*}
  \mathsf{Marked}(\vec{x}) \gets & 
                                   \mathsf{Type}(\vec{x}),\{a\}\in\vec{x},
                                   \mathsf{hasType}^k(a,\vec{y}), B \in \vec{x}, B \notin \vec{y} \\
  \mathsf{Marked}(\vec{x}) \gets &
                                   \mathsf{Type}(\vec{x}),\{a\}\in\vec{x},
                                   \mathsf{hasType}^k(a,\vec{y}),B \notin \vec{x}, B \in \vec{y}
\end{align*}

 Finally, since now fringe elements are exactly the constants of the
 input ABox, the rules in (IX) are replaced by a single constraint as
 follows:
 \[\gets \mathsf{Marked}(\vec{x}), \mathsf{RealizedType}(\vec{x})\]

 \smallskip
 
 Note that in the absence of closed predicates, with the above rules
 replacing (II), (III), (VII) and (IX), our rewriting does not use
 $\mathit{not}$ in rule bodies.  Based on these observations, we
 obtain the following:

 \begin{theorem} For a c-acyclic OMQ $Q=(\T,\q)$, where $\T$ is an
   $\ALCHOI$ TBox, we can build in polynomial time a query $(P_Q,q)$,
   where $P_Q$ is a positive disjunctive \datalog program such that
   $\cert(Q,\A) = \cert((P_Q,q), \A)$ for any given ABox $\A$ over the
   concept and role names occurring in $\T$. In addition, if $\T$ is
   an $\mathcal{ALCHI}$ TBox, then $P_Q$ has no occurrences of the
   $\neq$ predicate.
 \end{theorem}

 We note that $\neq$-free positive disjunctive programs are not
 expressive enough to capture instance queries $Q=(\T,\q)$ when $\T$ has
 nominals. This follows from the following observation. For any
 positive $\neq$-free program $P$ and a set of facts $F$, if $P\cup F$
 has a model, then also $P\cup F'$ has a model, were $F'$ is obtained
 from $F$ by renaming its constants with fresh ones that do not occur
 in $P\cup F$. However, this property cannot be recast to
 $\ALCHOI$. Take the TBox $\T=\{A\ISA \{a\}\}$ and observe that $\T$
 is consistent with respect to the ABox $\A_1=\{A(a)\}$, but is
 inconsistent with respect to the ABox $\A_2=\{A(b)\}$.

 \subsection{The Need for Two Constants}\label{sec:dis} We note that
 the rewriting presented in Section 6.2 uses two distinct constants
 (namely $0$ and $1$), which are ``introduced'' by means of two facts
 that are always present in the constructed program. We observe that
 if no constants are allowed in rules, a polynomial time rewriting
 into a \datalog program with negation or a positive disjunctive
 \datalog program does not exist even for $\mathcal{ALC}$ TBoxes in
 the absence of closed predicates, under  common assumptions in
 complexity theory. 
 This can be argued using the well-known fact that deciding
 $(\{A(c)\},\T) \models B(c)$, where $A,B$ are concept names and $\T$
 is an \ALC TBox, is an \exptime-hard problem. Assume a triple
 $A(c),B(c),\T$ as above, and suppose that from $\T$ we can compute in
 polynomial time a desired (polynomially sized) program $P_{\T}$ that
 does not use constants. Since $P_{\T}$ is a proper rewriting, it is
 the case that $(\{A(c)\},\T) \models B(c)$ iff $B(c)\in J$ for every
 stable model $J$ of $P_{\T} \cup \{A(c)\}$. Since $A(c)$ is the
 only input fact, the grounding of $P_{\T} \cup \{A(c)\}$ is of
 polynomial size. From the complexity of \datalog programs consisting
 of disjunctive positive rules, or consisting of non-disjunctive rules
 with negation under the stable model semantics, we obtain a co\np
 upper bound for testing $(\{A(c)\},\T) \models B(c)$. This
 contradicts the belief that $\exptime \not\subseteq $ co$\np$.


\section{Conclusions and Future Work}\label{sec:conc} 

In this paper, we have proposed a novel technique for rewriting
c-acyclic OMQs of the form $Q=(\T, \Sigma, \q)$, where $\T$ is an
$\ALCHOI$ TBox, $\Sigma$ is a (possibly empty) set of closed
predicates, and $\q$ is a CQ, into a polynomially-sized \datalog
program $P_Q$ with negation under the stable model semantics.  We have
also shown that if $Q$ has no closed predicates (i.e.,
$\Sigma = \emptyset$) we can obtain a positive disjunctive \datalog
program with the built-in inequality predicate. If nominals are not
present in the input TBox, the inequality predicate is unnecessary. To
our best knowledge, these are the first such rewritings that take
\emph{polynomial time} to be computed.

{Our rewriting establishes an interesting connection between two very
  prominent reasoning formalisms.  On the one hand, the OMQs we
  consider allow for very rich ontological reasoning, and cover many
  of the most popular DL constructs (indeed, $\ALCHOI$ has most of the
  constructs present in $\mathcal{SHOIN}$, the basis of the OWL DL
  standard \cite{DBLP:reference/sp/HorrocksP11}).  On the other hand,
  disjunctive \datalog with negation as failure is a very prominent
  and versatile language for common-sense reasoning and problem
  solving.  Our results show that the former can be effectively
  translated into the latter in polynomial time. To prove the
    correctness of this translation, we have used as an intermediate
    step a game-like characterization of the mentioned OMQs, which we
    believe is interesting in its own right.
The fact that 
previous translations for similar languages 
required exponential time, 
points to the fact that, although related, these two formalisms can express knowledge
in rather different ways.  
Given the differences in computational complexity of these
formalisms, it is natural that 
the translation results in a program that is inherently non-ground, 
and uses predicates whose arity is not bounded. 
The presence of closed predicates makes our OMQs particularly well suited for
 settings in which complete and incomplete data coexist, allowing for
 non-monotonic inferences that exploit the knowledge about partial
 completeness. 
The price to pay for this is that 
we  must target for the translation a variant of disjunctive \datalog with
 negation. However, the use of negation as failure in the program resulting
from our translation is rather limited.}

\paragraph{Extensions  and Future Work}
 We have presented our results for $\ALCHOI$, but they also apply to
$\mathcal{SHIO}$, using standard techniques to eliminate transitivity axioms (see, e.g., \cite{DBLP:journals/jar/HustadtMS07}). 
Moreover, the results can be easily generalized to \emph{DL-safe} rules of
\cite{DBLP:journals/ws/MotikSS05}. 
These queries are syntactically restricted to
ensure that the relevant variable assignments only map into individuals of the
input ABox. 
 \newt{We remark that our results also apply to other OMQs that can be reduced
   in polynomial time to the OMQs considered in this paper. 
For instance, some restricted forms of \emph{navigational queries} (like, for example,
nested regular path
  queries with at least one existentially quantified variable) can be
reduced to our OMQs by adding a polynomial number of inclusions to the TBox,
see for example \cite{DBLP:conf/kr/BienvenuCOS14}.}

Under common assumptions in complexity theory, our
translation cannot be generalized to 
CQs, while remaining polynomial. This is because query answering for
 a disjunctive program with negation  is in co\nexptimenp, but
OMQ answering 
 is \twoexptime-hard
already for the DLs $\mathcal{ALCI}$\cite{DBLP:conf/cade/Lutz08} or $\mathcal{ALCO}$ \cite{DBLP:conf/kr/NgoOS16}. 
Adapting the ideas in this work to CQs may be possible at the cost of an
exponential blow-up, but we believe it would be technically quite involved. 
An interesting task for future direction is to obtain a polynomial translation for
 $\mathcal{ALCHOIQ}$, which adds \emph{number restrictions} to
 $\ALCHOI$. 
Considering other query languages, such as variations of {regular path
  queries} \cite{DBLP:journals/iandc/CalvaneseEO14}, is also a compelling direction for future research.

 \subsection*{Acknowledgments}
 
This work was supported by the Austrian Science Fund (FWF) projects P30360, P30873, and W1255-N23.

\bibliographystyle{plainnat}
\bibliography{closed-to-datalog}

\clearpage
\section{Appendix} \label{sec:appendix}

\paragraph{\bf Proof of Proposition \ref{p:non-losing}}

We provide the missing, more technical, part of 
 the proof of Proposition \ref{p:non-losing}. 
We have already shown  that $\J$ is an extension of $\I_c$ and that $\J \models_{\Sigma} \A$.  To complete the proof of  soundness, it is left to show that $\J \models \T$, more precisely that $\J$ satisfies all inclusions of the forms $\mathbf{(N1)}$, $\mathbf{(N2)}$, $\mathbf{(N3)}$, and $\mathbf{(N4)}$ in $\T$. We show the claim for each such inclusion next.


\begin{enumerate}

\item[$\mathbf{(N1)}$] For inclusions of the form $B_1\AND \cdots \AND B_n \ISA
  B_{n+1}\OR\cdots \OR B_{k}$, by construction of $\J$
and by definition of non-losing $\str$, for all $w \in \dom{\J}$, 
$\type(w, \J) = \tail(w)$ and $\tail(w)\in \LC(\T,\Sigma,\I_c)$ -- that is, $\tail(w)$ satisfies all inclusions of type $\mathbf{(N1)}$. Hence $\J$ satisfies all inclusions of type $\mathbf{(N1)}$.

\item [$\mathbf{(N2)}$] Consider an inclusion $\Ax = A \ISA \exists r.A' \in
  \T$. We distinguish the following cases: 

\begin{itemize}
\item If \rin, that is there is some inclusion $r\suprr s$ with $s^{(-)}\in \Sigma$, then that $\J$ satisfies $\alpha$ is a direct consequence of the
  definition of a core, more precisely that $\I_c\models A \ISA \exists r.A'
  $, and the fact that $\J$ is an extension of $\I_c$. 
\item If \rnotin, let $w$ be an arbitrary object in $ A^\J$. We show that
  there exists a $v \in \dom{\J}$ such that $(w,v) \in r^\J$ and $v \in
  A'^{\J}$. By construction of $A^\J$ and by construction of
  $\run(\I_c,\str)$, $A \in \tail(w)$ holds and $w$ is a run that follows
  $\str$. Since $\str$ is a non-losing strategy $\str(\tail(w), \Ax)$ is
  defined, that is there exists a $\tau'$ over $\T$ such that $\str(\tail(w),
  \Ax) = \tau'$ and $A' \in \tau'$. We further distinguish the following
  cases:  
 \begin{itemize}
\item   
 If $w=a$ for an individual $a \in \dom{\I_c}$ and $\tau'$ is a c-type,
 then the assumption holds by definition of a strategy, that is there exists
 $(a,a')\in r^{\I_c}$ with $\tau'=(a',\I_c)$, and thus also $a' \in
 A'^{\J}$.
\item Otherwise, if $w=a$ and $\tau'$ is not a c-type, then we have  
   $v=a\Ax\tau' \in \run(\I_c,\str) = \dom{\J}$, 
and by construction either $(a,v) \in r^\J$ if $r$ is a role
 name, or $(v,a) \in (r^-)^\J$ if $r^-$ is a role name. Since  $A' \in
 \tail(a\Ax\tau')$, then $a\Ax\tau' \in A'^\J$ as desired. 
 \item The case when $w$ is a run
 of the form $a\ldots\Ax_\ell \tau_\ell$ and $\tau'$ is not a c-type is similar,
 that is, $w\Ax\tau' \in \run(\I_c,\str) = \dom{\J}$, and as $A' \in \tau'$,
 we have $w\Ax\tau' \in A'^\J$, and either $(w, w\Ax\tau') \in r^\J$ if $r$ is a role
 name, or
 $(w\Ax\tau',w) \in (r^-)^\J$ if $r^-$ is a role name. 

\item Finally, assume $w$ is a run of the form $a\ldots\Ax_\ell \tau_\ell$ and
  $\tau'$ is a c-type. It follows that $v =w\Ax\tau' 
  \in \srun(\I_c,\str)$, and hence by
  construction, $(w,a_{\tau'}) \in r^\J$ if $r$ is a role name, or
  $(a_{\tau'},w) \in (r^-)^\J$ if $r^-$ is a role name, 
where $a_{\tau'}$ is the chosen individual that realizes $\tau'$ in 
$\I_c$. Clearly, since $A' \in \tau'$,  we have $a_{\tau'} \in A'^\J$. 
\end{itemize}
\end{itemize}

\item[$\mathbf{(N3)}$]   Consider an inclusion $\alpha =   A_1\ISA
   \All{r}{A_2}$. To show $\J \models \alpha$, 
let $w$ be an arbitrary object in $A_1^\J$. By construction
   of $\J$, $w \in A_1^\J$ implies $A_1 \in \tail(w)$.  
Consider an arbitrary object $w' \in \dom{\J}$ such
 that $(w,w') \in r^\J$; note that the claim trivially holds for $w$ in case
 there is no such  $w'$. To show that $w' \in A_2^\J$, 
we distinguish the following cases.  

	\begin{itemize}
	 \item Both $w$ and $w'$ are  individuals from the core, namely $a$
           and $a'$, respectively. Since $\J$ is an extension of $\I_c$, then
           $(a,a') \in r^{\I_c}$. By definition of a core, $ \I_c \models
           A_1\ISA \All{r}{A_2}$ and, therefore, $a' \in A_2^{\I_c}$. Since 
$\J$ is an extension of $\I_c$, we have $a' \in A_2^{\J}$.
\end{itemize}

If at least $w$ or $w'$ is not an individual, then, by the definition of
$r^\J$, we have four possible cases: 
	  \smallskip
    	\begin{itemize}
	
	\item  $w$ is arbitrary, and 
	$w'$ is of the form $w\Ax_i\tau_i$ for 
  some inclusion $\Ax_i = A \ISA \exists r_i.A' \in \T$ with $r_i
\,{\not\in_\T}\, \Sigma$. We further distinguish two cases.
	
	\begin{itemize}
\item	If $r$ is a role name, then, by construction of $\J$, $r_i \suprr r$ follows. 
As $w\Ax_i\tau_i $ is a run that 
follows $\str$, 
$\tau_i$ must satisfy (C2). This together with the assumption 
 $A_1 \in \tail(w)$ and the fact that $r_i\suprr r$, imply that
$A_2 \in \tau_i$, and 
 $w' = w\Ax_i\tau_i \in A_2^\J$ follows. 

\item If $r^-$ is a role name, then  $(w\Ax_i\tau_i, w) \in (r^{-})^\J$. By
  construction of $\J$, it must be the case that $r_i^-\suprr r^-$, so 
$r_i\suprr r$ follows, and we can argue as above, using (C2) to conclude that 
$A_2 \in \tau_i$ and $w' = w\Ax_i\tau_i \in A_2^\J$.
\end{itemize}  

\item $w'$ is an arbitrary word, and $w$ is of the form $w'\Ax_i\tau_i \in
  \dom{\J}$, 
  for $\Ax_i$ an inclusion $A \ISA \exists r_i.A' \in \T$ with $r_i
  \,{\not\in_T}\, \Sigma$. We again distinguish two cases.
	\begin{itemize}
		\item	 If  $r$ is a role name, then by construction of $\J$,
                  $r_i^- \suprr r $. 
As 
 $w'\Ax_i\tau_i $ is a run that follows $\str$, 
then $A \in \tail(w')$, $A' \in \tau_i$, and 
$\tau_i$ must satisfy (C2). From $A_1 \in \tail(w)$ and $\tail(w)=
\tail(w'\Ax_i\tau_i)$ follows that  $A_1\in \tau_i$. The latter together with
$A' \in \tau_i$, $A \in \tail(w')$, and $r_i^-\suprr r$ imply that $A_2
\in \tail(w')$,
so $w' \in A_2^\J$. 
			\item If $r^-$ is a role name, then
                          $(w',w'\Ax_i\tau_i) \in (r^{-})^\J$. By construction
                          of $\J$, $r_i\suprr r^-$, so 
 $r_i^-\suprr r$ holds,
and we can argue as above 
that $w' \ in A_2^\J$. 
  \end{itemize}
	\item $w$ is an arbitrary word other than an individual, and $w'$ is
          an individual $a_{\tau_i}$, where 
$w\Ax_i\tau_i \in \srun(\I_c,\str)$ for some $\Ax_i = A \ISA \exists r_i.A'$
with $r_i \,{\not\in_T}\, \Sigma$, and $a_{\tau_i}$ realizes $\tau_i$ in
$\I_c$. 
We again have two cases.
    \begin{itemize} 				
   \item If $r$ is a role name, then $r_i \suprr r$. Since
$w\Ax_i\tau_i $ is a run that follows $\str$, 
$A \in \tail(w)$, $A' \in \tau_i$, and $\tau_i$ must satisfy (C2). The
     latter together with  $A_1 \in \tail(w)$ and $r_i\suprr r$ 
  imply that $A_2 \in \tau_i$. As 
    $\tau_i = \tail(a_{\tau_i})$, we have that 
$A_2 \in  \tail(a_{\tau_i})$ and $a_{\tau_i} \in A_2^\J$.

   \item If $r^-$ is a role name, 
  then  it is the case that $r_i^-\suprr r^-$ and    $r_i\suprr r$, 
 so we argue as above
that $w' =a_{\tau_i} \in A_2^\J$.
   \end{itemize}
	
	  \item $w$ is an individual $a_{\tau_i}$, while $w'$
  is not an individual, and
$w'\Ax_i\tau_i \in \srun(\I_c,\str)$ for some $\Ax_i = A \ISA \exists r_i.A'$
with $r_i \,{\not\in_\T}\, \Sigma$, and $a_{\tau_i}$ realizes $\tau_i$ in
$\I_c$.   We again consider two cases.
	  	
	  		\begin{itemize}
				  \item  If  $r$ is a role name, then
$r_i^- \suprr r $. As 
$w=a_{\tau_i} \in A_1^\J$ and 
$a_{\tau_i}$ is a run that follows \str, by construction of $\J$ we have
$A_1 \in \tail(a_{\tau_i})=\tau_i$. As 
$w'\Ax_i\tau_i $ is a run that follows $\str$, $A \in \tail(w')$, $A' \in
\tau_i$, and $\tau_i$ must satisfy (C2). The latter together with $A_1 \in
\tau_i$ and $r_i^-\sqsubseteq r$ imply that $A_2 \in \tail(w')$, hence $w' \in A_2^\J$.   

				\item If $r^-$ is a role name,  
                              then it is the case that  $r_i^-\suprr r^-$ and 
$r_i\suprr r $, and therefore we
argue as above that $w' \in A_2^\J$.
	\end{itemize}

	\end{itemize}

\item [$\mathbf{(N4)}$] Finally,  consider an inclusion $\alpha = r \sqsubseteq s \in \T$. 
For pairs of individuals, $(a,a') \in r^\J $ implies $(a,a') \in s^\J$ 
because $\J$ is an extension of $\I_c$ and $\I_c \models r\sqsubseteq s$. For all other pairs of objects, it is not hard to verify that
$\J \models r\sqsubseteq s$ is guaranteed by the construction of
$\J$, in particular the fact that $p^\J$ is closed under the role inclusions in
$\T$, and 
the fact that,  
  Otherwise, assume $(w,v) \in r^\J$,  for an arbitrary pair of an object and its child, where at least one is not an individual. If $s$ is a role name $p$, then that $(w,v)$ belongs to $p^\J$ is ensured by construction of $\J$;  otherwise
 if $r \sqsubseteq p^-$ is in $\T$ for a role name $p$, then $r^- \sqsubseteq
 p$ is also in $\T$.

\end{enumerate}


\end{document}